%% file: colt.tex
\documentclass[arxiv,12pt]{colt2021} 

\usepackage[utf8]{inputenc}
\usepackage{mathextra_colt}
\usepackage{algorithm}
\usepackage{enumitem}
\usepackage{tcolorbox}
\usepackage{tabulary}
\usepackage{booktabs}
\usepackage{longtable}
\usepackage[strict]{changepage}
\usepackage{setspace}
\usepackage[bottom]{footmisc}
\usepackage{wrapfig}

\makeatletter
\let\Ginclude@graphics\@org@Ginclude@graphics
\makeatother

\title[Asymptotically Optimal Information-Directed Sampling]{Asymptotically Optimal Information-Directed Sampling}
\usepackage{times}

\coltauthor{%
	\Name{Johannes Kirschner} \Email{jkirschner@inf.ethz.ch}\\
	\addr ETH Zurich, Department of Computer Science {$^*$}
	\AND
	\Name{Tor Lattimore} \Email{lattimore@google.com}\\
	\addr DeepMind%
	\AND
	\Name{Claire Vernade} \Email{vernade@google.com}\\
	\addr DeepMind%
	\AND
	\Name{Csaba Szepesv\'ari} \Email{szepesva@ualberta.ca}\\
	\addr University of Alberta \& Deepmind
}
\makeatletter
\renewcommand*{\@titlefoot}{\scriptsize\copyright\space\@jmlryear
	\space\@jmlr@authors.\hfill $^*$Research conducted during an internship at DeepMind.
	\@reprint
}
\makeatother

\let\todo\undefined  
\usepackage[textsize=tiny,disable]{todonotes}

\newtcbox{\entoure}[1][red]{on line,
	arc=3pt,colback=#1!10!white,colframe=#1!50!black,
	before upper={\rule[-3pt]{0pt}{10pt}},boxrule=1pt,
	boxsep=0pt,left=2pt,right=2pt,top=1pt,bottom=-1pt}

\newcommand{\ip}[1]{\langle#1\rangle}

\newcommand{\htheta}{\hat \theta}
\newcommand{\hDelta}{\hat \Delta}

\newcommand{\poly}{\text{poly}}
\newcommand{\Dmin}{\Delta_{\min}}

\newcommand{\hnu}{\hat \nu}
\newcommand{\deq}{\triangleq}

\newcommand{\IDS}{\text{IDS}}
\newcommand{\TS}{\text{TS}}
\newcommand{\UCB}{\textsc{UCB}}
\newcommand{\LinUCB}{\textsc{LinUCB}}
\newcommand{\LinTS}{\textsc{LinTS}}
\newcommand{\solid}{\textsc{solid}}

\newcommand{\IA}{\hH}
\newcommand{\IACELL}{\cC}
\newcommand{\IAUCBCELL}{{\cC\text{-UCB}}}
\newcommand{\IAUCB}{{\hH\text{-UCB}}}

\newcommand{\IVAR}{\text{VAR}}

\newcommand{\IMI}{\text{MI}}

\newcommand{\diam}{\text{diam}}
\newcommand{\y}{y}
\newcommand{\thetaopt}{\theta^*}
\newcommand{\betass}{\beta_{s,s^{{2}}}}

\newcommand{\tsum}{\textstyle\sum}
\newcommand{\B}{\Bigg}

\newcommand{\hh}[2]{\hH^{#1}_{#2}}

\newcommand{\sdfrac}[2]{\mbox{\small$\displaystyle\frac{#1}{#2}$}}

\newtheorem*{lemma*}{Lemma}

\SetCommentSty{mycommfont}
\let\oldnl\nl
\newcommand{\nonl}{\renewcommand{\nl}{\let\nl\oldnl}}
\newcommand{\lrate}[1][s]{\min_{l \leq #1} m_l^{-1/2} \log(k)}
\newenvironment{proofof}[1]%
{%
 \par\noindent{\bfseries\upshape Proof of #1\ }%
}%
{\jmlrQED}

\begin{document}

\maketitle

\begin{abstract}
	We introduce a simple and efficient algorithm for stochastic linear bandits with finitely many actions that is asymptotically optimal and (nearly) worst-case optimal in finite time. The approach is based on the frequentist information-directed sampling (IDS) framework, with a surrogate for the information gain that is informed by the optimization problem that defines the asymptotic lower bound. Our analysis sheds light on how IDS balances the trade-off between regret and information and uncovers a surprising connection between the recently proposed primal-dual methods and the IDS algorithm. We demonstrate empirically that IDS is competitive with UCB in finite-time, and can be significantly better in the asymptotic regime.
\end{abstract}\todo[color=red!50!white,size=\scriptsize]{Disable Notes.}

\input{parts/introduction.tex}

\input{parts/ids.tex}

\input{parts/ids-outline.tex}

\input{parts/ids-experiments.tex}

\input{parts/conclusion.tex}

\section*{Acknowledgements}
Csaba Szepesv\'ari gratefully acknowledges funding from the Canada CIFAR AI Chairs Program, Amii and NSERC.

\bibliography{references.bib}
\newpage

\appendix 
\input{parts/notation.tex}
\input{parts/appendix-ids.tex}
\input{parts/appendix-ratio.tex}
\input{parts/appendix-infogain.tex}

\input{parts/appendix-asymptotic.tex}
\input{parts/appendix-technical.tex}

\input{parts/covering.tex}
\input{parts/appendix-bayesian.tex}

\input{parts/appendix-experiments}

\end{document}

%% file: parts/introduction.tex
\section{Introduction}

The stochastic linear bandit problem is an iterative game between a learner and an environment played over $n$ rounds.
In each round $t$, the learner chooses an action (or arm) $x_t$ from a finite set of actions $\xX \subset \RR^d$ and observes a noisy reward $y_t = \ip{x_t, \thetaopt} + \epsilon_t$ where $\thetaopt \in \RR^d$ is an unknown
parameter vector and $\epsilon_t$ is zero-mean noise.
The learner's goal is to maximize the expected cumulative reward or, equivalently, to minimize the expected regret, which is defined by 
\begin{align}
R_n(\pi, \thetaopt) = \EE\left[\max_{x \in \xX} \sum_{t=1}^n \ip{x - x_t, \thetaopt}\right]\,, \label{eq:regret}
\end{align}
where $\pi$ is the policy mapping sequences of action/reward pairs to distributions over actions in $\xX$ and the expectation is over 
the randomness in the policy and the rewards.
Unlike in the multi-armed bandit setting, the linear structure allows the learner to estimate the reward of an action without directly observing it. 
In particular, the learner might play an action that it knows to be suboptimal in order to most efficiently identify the optimal action.

The \emph{worst-case regret} $R_n(\pi) = \sup_{\theta \in \mM} R_n(\pi, \theta)$ measures the performance of a policy on an adversarially chosen parameter $\theta$ in a class of models $\mM$. On the other hand, for a fixed instance $\thetaopt$, an algorithm can perform much better than the worst-case regret $R_n(\pi)$ suggests, and achieving the optimal instance-dependent regret $R_n(\pi, \thetaopt)$ is therefore of significant interest. On a large horizon, the optimal instance-dependent regret, 
or \emph{asymptotic regret}, is characterized by a convex program, that optimizes the allocated proportion of plays to each action to minimize the regret, subject to the constraint that the policy gathers enough information to infer the best action \citep{graves1997asymptotically}.

The optimal worst-case regret rate (up to logarithmic factors) is achieved by various algorithms, including adaptations of the upper confidence bound (UCB) algorithm \citep{auer2002confidencebounds,dani2008stochastic,Abbasi2011improved} and the information-directed sampling (IDS) approach \citep{Russo2014learning,Kirschner2018}. A conservative version of Thompson sampling is suboptimal by a factor of $\sqrt{d}$ and logarithmic factors \citep{agrawal2013thompson}. On the other hand, achieving optimal asymptotic regret has proven to be more challenging. \citet{lattimore2017end} showed that algorithms based on optimism or Thompson sampling are not asymptotically optimal in the linear setting. They propose an approach based on the explore-then-commit framework that computes an estimate of the optimal allocation and updates the allocation to match the predicted target. \citet{combes2017minimal} follow a similar plan for the structured bandit setting, which includes the linear setting as a special case. This idea was subsequently extended to the contextual setting by \citet{hao2019adaptive}. Unfortunately these algorithms are not at all practical and do not enjoy reasonable minimax regret.
More recently, \citet{jung2020crushoptimisim} refined this technique in the structured setting with a finite model class to avoid forced exploration and the knowledge of the horizon. Similarly, \cite{van2020optimal} use a dual formulation of the lower bound to devise an algorithm that achieves the optimal asymptotic regret up to a constant, and avoids re-solving for the predicted optimal allocation at every round. \cite{degenne2020structure} take a different approach and translate the Lagrangian of the lower bound into a fictitious two-player game, where the saddle point corresponds to the asymptotic regret. Using tools from online convex optimization \citep{hazan2016introduction,orabona2019modern}, this leads to a family of asymptotically optimal algorithms, which incrementally update the allocation in each round based on primal-dual updates on the Lagrangian of the lower bound. Another primal-dual method is by \cite{tirinzoni2020asymptotically}, which unlike previous methods is both worst-case and asymptotically optimal and also applies to the contextual case. We explain how IDS relates to primal-dual methods in Section~\ref{ss:primal-dual}. Finally, \citet{wagenmaker2020} combine optimal experimental design with a phased elimination-style algorithm to derive finite-time guarantees that scale with the Gaussian width of the action set.

\paragraph{Contributions} Our main contribution is new \emph{conceptual insights} into information-directed sampling (IDS). We show that with an appropriate choice of the information gain, IDS performs \emph{primal-dual updates} on the Lagrangian of the lower bound. The proposed version of IDS for the linear bandit setting is (nearly) \emph{worst-case optimal} in finite time, satisfies an explicit \emph{gap-dependent logarithmic regret} bound and is \emph{asymptotically optimal}. All regret bounds are on \emph{frequentist expected regret} and our analysis is relatively simple, avoiding all but one high-probability bound. The asymptotic analysis uncovers a connection between IDS and recently proposed primal-dual methods \citep{degenne2020structure,tirinzoni2020asymptotically}. Moreover, our choice of information gain function approximates the variance based information gain proposed by \citep{Russo2014learning} in the Bayesian setting. \looseness=-1

\paragraph{Notation} The real numbers are $\RR$ and $\RR_{\geq 0}$ denotes the positive orthant.
The standard Euclidean norm is $\| \cdot \|$ and the Euclidean inner product is $\ip{\cdot, \cdot}$. The Euclidean basis in $\RR^m$ is $e_1, \dots, e_m$. The identity matrix in $\RR^{d\times d}$ is $\mathbf{1}_d$. 
The diameter of a set $\xX \subset \RR^d$ is $\diam(\xX) = \sup_{x,y \in \xX} \norm{x - y}$.
For a positive (semi-)definite, symmetric matrix $A \in \RR^{d \times d}$ and a vector $v \in \RR^d$, the associated matrix (semi-)norm is $\| v \|_A^2 = \ip{v, Av}$. $\sP(\xX)$ is the set of probability measures on a finite set $\xX$. Where convenient, we use vector notation, including inner products to denote evaluation of functions $F \in \RR^\xX$, for example $F(x) = \ip{e_x, F}$. Functions $F \in \RR^{\xX}$ are extended linearly to distributions $\mu \in \sP(\xX)$ to denote the expectation $F(\mu) = \ip{\mu, F} = \sum_{x \in \xX} f(x) \mu(x)$. In this context, we also use $e_x$ for the Dirac measure on $x \in \xX$. 
The reader may refer  to Appendix \ref{app:notation} for a summary of notation. \looseness=-1

\subsection{Setting}
Let $\xX \subset \RR^d$ be a finite set of $k$ actions. We assume that $\xX$ spans $\RR^d$ and $\diam(\xX) \leq 1$. 
Denote by $\thetaopt \in \mM$ an unknown parameter vector, where $\mM \subset \RR^d$ is a known convex polytope with $\diam(\mM) \leq 1$.
In each round $t=1,\dots, n$, the learner chooses a distribution $\mu_t$ over $\xX$. Then $x_t$ is sampled from $\mu_t$ and the learner observes
the reward $\y_t = \ip{x_t, \theta} + \epsilon_t$ where $\epsilon_t$ is sampled independently from a Gaussian with zero mean and unit variance. 
All our upper bounds hold without modification for conditionally $1$-subgaussian noise.
The objective is to minimize the expected cumulative regret $R_n = R_n(\pi, \thetaopt)$ defined in Eq.~\eqref{eq:regret}, where $\pi = (\mu_t)_{t=1}^n$ is the policy chosen by the learner. The dependency of the regret on $\thetaopt$ and $\pi$ is mostly omitted when there is no ambiguity. The expectation conditioned on previous observations is $\EE_s[\,\cdot\,] = \EE[\,\cdot\,|(x_l, y_l)_{l=1}^{s-1}]$.
In line with all previous work focusing on the asymptotic setting, we assume that the optimal action $x^* = x^*(\thetaopt)= \argmax_{x \in \xX} \ip{x, \thetaopt}$ is unique. Eliminating this assumption is left as a delicate and possibly non-trivial challenge for the future. The sub-optimality gap of an action $x \in \xX$ is $\Delta(x) = \ip{x^* - x, \thetaopt}$ and $\Dmin = \min_{x \neq x^*} \Delta(x)$ denotes the smallest non-zero gap.
For actions $x,z \in \xX$, we denote by $\hh{z}{x} = \{\nu \in \mM : \ip{x - z, \nu} \geq 0\}$ the (convex) set of parameters where the reward of $x$ is at least the reward of $z$. The set of \emph{alternative parameters} is $\cC^* = \cup_{x \neq x^*} \hh{x^*}{x}$.

\paragraph{Asymptotic Lower Bound}

For an allocation $\alpha \in \RR_{\geq 0}^\xX$ over actions we define the associated covariance matrix $V(\alpha) = \sum_{x \in \xX} \alpha(x) x x^\T$.
Let $c^*$ be the solution to the following convex program,
\begin{align}
	c^* \deq \inf_{\alpha \in \RR_{\geq 0}^\xX} \sum_{x \in \xX} \alpha(x)\ip{x^* - x, \theta^*} \qquad \text{s.t.} \qquad  \min_{\nu \in \cC^*} \tfrac{1}{2} \|\nu - \thetaopt\|_{V(\alpha)}^2 \geq 1\,. \label{eq:lower-linear}
\end{align}
The optimization minimizes the regret over (unbounded) allocations $\alpha$ that collect sufficient statistical evidence to reject all parameters $\nu \in \cC^*$ for which an action $x \neq x^*$ is optimal. Note that for a fixed $\nu \in \RR^d$, the constraints are linear in the allocation, $\|\nu - \theta\|_{V(\alpha)}^2 = \sum_{x \in \xX} \alpha(x) \ip{\nu - \theta, x}^2$. 
The next lemma is a well-known result, which relates the asymptotic regret to the solution of \eqref{eq:lower-linear}. 
A policy $\pi$ is called \emph{consistent} if for all $\theta \in \mM$ and $p > 0$ it holds that $R_n(\theta, \pi) = o(n^p)$.
Assuming consistency is required to rule out policies that are defined to always play a fixed action $x^*$, which incurs zero regret when $x^*$ is indeed optimal, but linear regret on other instances.

\begin{theorem}[\citet{graves1997asymptotically,combes2017minimal}]
	Any consistent algorithm  $\pi$ for the linear bandit setting with Gaussian noise has regret $R_n(\thetaopt, \pi)$ at least \[\liminf_{n\rightarrow \infty} \frac{R_n(\thetaopt, \pi)}{\log(n)} \geq c^*(\thetaopt) \,.\]
\end{theorem}


%% file: parts/ids.tex
\input{parts/ids-anytime.tex}

\section{Asymptotically Optimal Information-Directed Sampling} \label{sec:ids}

The information-directed sampling (IDS) principle was introduced by \cite{Russo2014learning} in the Bayesian setting.
Our work is based on the frequentist version of this approach, developed by \cite{Kirschner2018}.
The central idea is to compute a distribution over the actions that optimizes the following trade-off between a gap estimate $\hat \Delta_s(x)$ and an information gain $I_s(x)$, defined at step $s \geq 1$ for each $x \in \xX$:
\begin{align}
	\mu_s = \argmin_{\mu \in \sP(\xX)} \left\{\Psi_s(\mu) \deq \frac{\hDelta_{s}(\mu)^2}{I_{s}(\mu)}\right\}\,. \label{eq:ids-def}
\end{align}
Intuitively, this objective requires to sample actions that have either small regret or large information gain. 
The \emph{information ratio} $\Psi_t$ is a convex function of the distribution \citep[Prop.~6]{Russo2014learning} and can be minimized efficiently as we explain below. In \emph{exploration rounds}, indexed by $s$, IDS samples the action $x_s$ from the IDS distribution $\mu_s$. Otherwise, in \emph{exploitation rounds},  $x^*$ is identified with high probability, and the algorithm plays the action it believes to be optimal, denoted by $\hat x_s$ (where $s$ is the index of the last exploration round). 
The interaction with the environment, described in Algorithm~\ref{alg:ids-anytime}, is over rounds $t=1,\dots,n$ on a \emph{horizon} $n$, which is unknown a priori.
Exploration rounds are counted separately by $s$, inducing an implicit mapping $s \mapsto t_s \leq t$. The number of exploration rounds up to time $t$ is $s_t$. 
We refer to $s$ and $t$ as \emph{local} and \emph{global time} respectively, 
and to $s_n$ as the \emph{effective horizon}.
 The convention is that an $s$-index refers to the local time quantities, whereas a $t$-index refers to global time quantities.
For example, the action chosen in exploration round $s$ at global time $t_s$ is $x_s$ and the observed reward is $y_s$. Similarly, an action $x_s$ at local time $s$ has a global time correspondence $x_t = x_{t_s}$. 

\paragraph{Gap Estimates}

All estimated quantities are defined using data collected in exploration rounds, whereas observation data from exploitation rounds is discarded. To justify this choice intuitively, note that with high probability, in exploration rounds the algorithm samples the optimal action $x^*$, thereby accumulating exponentially more data points on the optimal actions compared to suboptimal actions.  Ignoring data from exploitation rounds leads to a much more balanced data set. 

Let $\hat \theta_s \deq V_s^{-1} \sum_{i=1}^{s-1} x_i y_i$ be the regularized least squares estimator with covariance matrix $V_s \deq \sum_{i=1}^{s-1} x_i x_i^\T + \mathbf{1}_d$, computed with data $(x_1, y_1), \dots, (x_{s-1},y_{s-1})$. The \emph{empirically best action}  is $\hat x_s \deq \argmax_{x \in \xX} \ip{x, \hat \theta_{s}}$.
We assume that the learner has a \emph{concentration coefficient} $\beta_{s,1/\delta}$ that satisfies 
\begin{align} 
\PP[ \exists\, s \geq 1 \text{ with }\|\htheta_s - \thetaopt\|_{V_s}^2 \geq \beta_{s,1/\delta}] \leq \delta \,.  \label{eq:confidence-coefficient}
\end{align}
For concreteness, we use the choice derived by \cite{Abbasi2011improved}, which is
\begin{align}
\beta_{s,1/\delta}^{1/2} \deq \sqrt{2\log \delta^{-1} + \log \det(V_s)} + 1 \,.
\end{align}
The reader might worry about the log determinant term, which is known to create an asymptotically suboptimal dependence on the dimension, and can be improved with a different choice of the confidence coefficient \citep{lattimore2017end}. 
Since $\beta_{s,1/\delta} = 2 \log \frac{1}{\delta} + \oO(d\log(s))$, we circumvent this shortcoming by limiting the amount of data the algorithm collects to $s_n = \oO\big(\poly(\log(n)\big)$, which implies $\beta_{s_n, 1/\delta} = 2\log\frac{1}{\delta} + \oO(d \log \log(n))$. We also exploit this property for other steps in the analysis, but it is unclear whether or not it is essential.

For all $z \neq \hat x_s$, let $\hat \nu_s(z) = \argmin_{\nu \in \hh{\hat x_s}{z}} \|\nu - \hat \theta_s\|_{V_s}^2$ be the closest parameter to $\hat \theta_s$ in $V_s$-norm for which $z$ is better than $\hat x_s$. This is a strongly convex objective over the convex set $\hh{\hat x_s}{z}$, hence $\hat \nu_s(z)$ can be computed efficiently. In practice, we can drop the constraints on the parameter set (i.e.~set $\mM=\RR^d$), in which case $\hat \nu_s(z)$ can be computed in closed form, see \eqref{eq:nu-closed-form} below. Exploitation rounds are defined by the  \emph{exploitation condition},
\begin{align}
	m_s \deq \tfrac{1}{2} \min_{x \neq \hat x_s} \|\hat \nu_s(x) - \hat \theta_s\|_{V_s}^2 \geq \tfrac{1}{2}\beta_{s, t\log(t)}\,, \tag{E}\label{eq:exploitation}
\end{align}
which guarantees that with confidence level $\beta_{s, t\log(t)}$ there exists no plausible alternative parameter $\nu \neq \hat \theta_s$, such that an action $x \neq \hat x_s$ is optimal for $\nu$. 
At local time $s$, the \emph{gap estimate} is
\begin{align*}
\hat \Delta_{s}(x) \deq \max_{z \in \xX} \ip{z - x, \hat \theta_{s}} + \betass^{1/2}\|z\|_{V_{s}^{-1}}\,.
\end{align*}
Note that we use a different confidence level in the definition of the gap estimate, and in fact the only explicit dependence on the global time $t$ is in the exploitation condition. The gap estimate is an upper bound on the true gap, provided $\hat \theta_s$ is well concentrated, i.e.\ $\|\hat \theta_{s} - \theta^*\|_{V_{s}}^2 \leq \beta_{s,s^2}$,
\begin{align}
\Delta(x) \leq  \max_{y \in \xX} \ip{y, \hat \theta_{s}} + \betass^{1/2}\|y\|_{V_{s}^{-1}} - (\ip{x, \hat \theta_{s}} - \betass^{1/2}\|x\|_{V_{s}^{-1}}) \leq 2 \hat \Delta_{s}(x)\,. \label{eq:gap-bound}
\end{align}
The first inequality follows from the definition of the confidence scores, and the second inequality uses $\smash{\hat \Delta_s(x) \geq \beta_{s,s^2}^{1/2}\|x\|_{V_s^{-1}}}$. The gap estimate of the empirically best action $\hat x_s$ is $\delta_s \deq \hat \Delta_{s}(\hat{x}_{s})$. Importantly, the gap estimate can be written as
$\hDelta_{s}(x) = \ip{\hat x_{s} - x, \htheta_{s}} + \delta_{s}$,
and therefore we also refer to $\delta_s$ as the \emph{estimation error}. The UCB action is $\smash{x_s^\UCB \deq \argmax_{x \in \xX} \ip{x, \hat \theta_s} + \beta_{s,s^2}^{1/2} \|x\|_{V_s^{-1}}}$.

\paragraph{Information Gain}
Recall that $\hat \nu_s(z) = \argmin_{\nu \in \hh{\hat x_s}{z}} \|\nu - \hat \theta_s\|_{V_s}^2$ is the closest alternative parameter to $\hat \theta_s$ in $V_s$-norm for which $\hat x_s$ is not optimal. The  \emph{information gain} is set to
\begin{align}\label{eq:I-def}
	I_{s}(x) \deq \sdfrac{1}{2}\sum_{z \neq \hat{x}_{s}} q_s(z) \left(|\ip{\hat \nu_{s}(z) - \htheta_{s},x}| + \beta_{s,s^2}^{1/2}\|x\|_{V_s^{-1}}\right)^2\,,
\end{align}
where the mixing distribution $q_s \in \pP(\xX)$ is defined so that
\begin{align}
  q_s(z) \propto
  \begin{cases}
    0 & \text{if } z = \hat x_s \\
	\exp\left(-\tfrac{\eta_s}{2}  \|\hat \nu_{s}(z) - \hat \theta_{s}\|_{V_{s}}^2\right) & \text{otherwise}\,.
  \end{cases}
\end{align}
The learning rate is $\smash{\eta_s \deq \lrate}$, where $m_s \deq \frac{1}{2}\min_{z \neq \hat x_s} \|\hat \nu_s(z) - \hat \theta_s\|_{V_s}^2$. The weights $q_s$ can be interpreted as a soft-min approximation of the minimum constraint value where the learning rate controls the lower order term (Lemma \ref{lem:softmin}),
\begin{align}
	m_s \leq \sdfrac{1}{2}\sum_{z \neq \hat x_s} q_{s}(z) \|\hat \nu_{s}(z) - \hat \theta_{s}\|_{V_{s}}^2 \leq m_s + \frac{\log(k)}{\eta_s}\,. \label{eq:softmin}
\end{align}


\paragraph{Computational Complexity} 
There are three kinds of operations in the algorithm. First, using elementary matrix operations, we can update $V_s^{-1}$, $\det(V_s)$ and $\hat \theta_s$ incrementally, and note that the $s$-index terms only need to be updated 
after exploration rounds. It can be checked that $\oO(k d^2 s_n)$ operations are needed over all $n$ rounds to compute this part. Second, the IDS distribution \eqref{eq:ids-def} is defined as a minimizer of the convex objective $\Psi_s(\mu)$ and always admits a solution supported on two actions \citep{Russo2014learning}, 
see Lemma \ref{lem:ids-distr}. Hence, we can obtain the IDS distribution by computing the optimal trade-off between all $\oO(k^2)$ pairs of actions (Lemma \ref{lem:two-action-ids}). A closer inspection of the regret bounds reveals that it always suffices to optimize the trade-off between the greedy action $\hat x_s$ and some other (informative) action, which reduces the computational complexity to $\oO(k)$. Third, the optimization problem that defines the alternative parameters $\hat \nu_s(z)$ is a quadratic program with $d$ variables and linear constraints $\ip{\hat \nu_s(z), z - \hat x_s} \geq 0$ and $\hat \nu_s(z) \in \mM$. Such optimization problems can be solved efficiently in practice and in $O(l d^3)$ time in the worst case for model sets $\mM$ with $l$ constraints.
Note, the analysis suggests that we can tolerate an additive numerical error on the information gain of order $\oO(s^{-2})$. In practice, we can drop the constraints on $\mM$, in which case\looseness=-1
\begin{align}
	\hat \nu_s(z) = \hat \theta_s - \tfrac{\ip{\hat \theta_s, \hat x_s - z}}{\|\hat x_s - z\|_{V_s^{-1}}^2} V_s^{-1}(\hat x_s - z)\,.\label{eq:nu-closed-form}
\end{align}
With these improvements, the overall computation complexity is $\oO(n + k d^2 s_n)$ over $n$ rounds, where the linear term comes from checking whether to explore or exploit. This can be improved, by simply computing after each exploration round when the next exploration round will occur.

%% file: parts/ids-anytime.tex
\LinesNumbered
\RestyleAlgo{ruled}
\begin{algorithm2e}[t]
	\DontPrintSemicolon
	\SetAlgoVlined
	\SetAlgoNoLine
	\SetAlgoNoEnd
	\caption{Asymptotically Optimal Information-Directed Sampling} \label{alg:ids-anytime}
	$s \gets 1$ \;
	\For{$t=1,2,3, \dots$}{
		$V_s \gets \sum_{i=1}^{s-1} x_{i} x_{i}^\T + \mathbf{1}_d$\;
		$\hat \theta_s \gets V_{s}^{-1}\sum_{i=1}^{s-1} x_{i} y_{i}$ \tcp*{least squares estimate}
		$\hat x_s \gets \argmax_{x \in \xX} \ip{x, \hat \theta_s}$\tcp*{empirically best action}
		$\beta_{s,1/\delta} \gets (\sqrt{2 \log \delta^{-1} + \log \det (V_s)} + 1)^2$ \; 
		$\hat \Delta_{s}(x) \gets  \big(\max_{z \in \xX} \ip{z, \htheta_s} + \beta_{s,s^2}^{1/2} \|z\|_{V_s^{-1}}\big) - \ip{x, \hat \theta_s}$ \tcp*{gap estimates}
		$\hat \nu_s(z) \gets \argmin_{\nu \in \hh{\hat x_s}{z}} \|\nu - \hat \theta_s\|_{V_s}^2$ \tcp*{see Eq.~\eqref{eq:nu-closed-form}}
		$m_s \gets \min_{z \neq \hat x_s} \tfrac{1}{2} \|\hat \nu_s(z) - \hat \theta_s\|_{V_s}^2$\;
		$\eta_{s} \gets \lrate$ \;
		$q_{s}(z) \gets \exp(-\eta_{s} \|\hat \nu_s(z) - \hat \theta_s\|_{V_s}^2)$\;
		$I_{s}(x) \gets \tfrac{1}{2}\sum_{z \neq \hat x_s} q_{s}(z) \big(|\ip{\hat \nu_s(z) - \hat \theta_s, x}| + \beta_{s,s^2}^{1/2}\; \|x\|_{V_s^{-1}}\big)^2$ \tcp*{information gain$^\dagger$}
		\If{$m_s \geq \tfrac{1}{2}\beta_{s, t \log(t)}$ }{ \label{line:exploit}
			Choose $\hat x_{s}$ \tcp*{exploitation (disregard data)}
		}
		\Else{			
		$\mu_s \gets \argmin_{\mu \in \sP(\xX)} \frac{\hat \Delta_{s}(\mu)^2}{I_{s}(\mu)}$ \tcp*{IDS distribution}
		Sample $x_s \sim \mu_s$, observe $\y_s = \ip{x_s, \thetaopt} + \epsilon_s$ \;
		$s \gets s+1$ \tcp*{exploration step counter}
		}
	}
	\vspace{4pt}
	\nonl \rule{200pt}{0.4pt}\\
	\vspace{2pt}
	\nonl \footnotesize{$^\dagger$ For the analysis, we normalize the $q$-weights, but this is not necessary to compute the IDS distribution.}
\end{algorithm2e}

%% file: parts/ids-outline.tex
\subsection{Regret Bounds}\label{ss:regret-outline}

The regret bounds for Algorithm \ref{alg:ids-anytime} come in three flavours. In Theorem \ref{thm:regret-worstcase}, we show a (nearly) optimal worst-case regret bound of $R_n \leq \oO(d \sqrt{n} \log(n))$. Second, using a gap-dependent bound on the information ratio, in Theorem \ref{thm:regret-log} we show a gap-dependent regret bound of $R_n \leq \oO\big(d^3 \Delta_{\min}^{-1} \log(n)^2\big)$. Besides universal constants, the $\oO$-notation in the bound only depends on the norm of action features and the parameter. Last, in Theorem \ref{thm:regret-asymptotic} we show that the proposed algorithm is asymptotically optimal, that is $R_n \leq c^* \log(n) + o(\log(n))$. In contrast to the previous bound, here the lower order terms depend exponentially on problem-dependent quantities such as $\Delta_{\min}^{-1}$.

\begin{theorem}[Worst-case regret]\label{thm:regret-worstcase} The regret of Algorithm \ref{alg:ids-anytime} is bounded by
	\begin{align*}
	R_n \leq \oO\big(d \sqrt{n} \log(n)\big)\,.
	\end{align*}
\end{theorem}
The result matches the best known bound for LinUCB and is optimal up to the logarithmic factor when $k$ is (exponentially) large.
On the other hand, when $k$ is small, our bound is worse
than basic elimination algorithms that achieve $R_n \leq \oO(\sqrt{\log(k) d n} )$ \citep[\S23]{Lattimore2019}.

\begin{proof}
	Define $\beta_s = \|\hat \theta_s - \thetaopt\|_{V_s}^2$ and let $B_s = \chf{\beta_s \leq \beta_{s, s^2}}$. By Lemma \ref{lem:exploration-regret} and \eqref{eq:gap-bound}, we have
	\begin{align*}
		\EE[R_n] \leq 2 \EE\left[\sum_{s=1}^{s_n} \hat \Delta_s(x_s)B_s\right] + \oO(\log \log(n))\,,
	\end{align*}
	where the $\oO$-notation only hides a bound on the largest gap, $\hat \Delta(x_s) \leq 1$. Similar to the standard IDS analysis \citep{Russo2014learning}, we bound the expected regret,
	\begin{align*}
	\EE\left[\sum_{s=1}^{s_n} \hat \Delta_s(x_s)B_s\right] 
  &= \EE\left[\sum_{s=1}^{s_n} \sqrt{\Psi_s(\mu_s) I_s(\mu_s) B_s}\right]
  \leq  \sqrt{\EE\left[\sum_{s=1}^{s_n} \Psi_s(\mu_s)B_s\right] \EE\left[\sum_{s=1}^{s_n} I_s(x_s)B_s\right]}\,,
	\end{align*}
  where the equality follows from the tower rule $\EE[\hat \Delta_s(x_s) B_s] = \EE[\EE_s[\hat \Delta_s(x_s) B_s]] = \EE[\hat \Delta_s(\mu_s) B_s]$ and the definition of the information ratio.
  The second inequality follows from Cauchy-Schwarz and another application of the tower rule.
  To complete the proof, we show that $\Psi_s(\mu_s) \leq 2$ and bound the total information gain, $\gamma_n =\sum_{s=1}^{s_n} I_s(x_s) \leq \oO(d^2 \log(n)^2)$.
  Since $\mu_s$ is chosen by IDS to minimize $\Psi_s$,
	\begin{align}
		\Psi_s(\mu_s) = \min_{\mu \in \sP(\xX)} \Psi_s(\mu) \leq \frac{\hat \Delta_s(x_s^\UCB)^2}{I_s(x_s^\UCB)} \leq 2\,. \label{eq:ratio-worst-case}
	\end{align}
	The last inequality follows from the fact that $\hat \Delta_s(x_s^\UCB) = \beta_{s,s^2}^{1/2} \|x_s^\UCB\|_{V_s^{-1}}$ and bounding 
  \begin{align*}
	I_s(x_s^\UCB) = \sdfrac{1}{2}\sum_{z \neq \hat x_s} q_s(z) \big(|\ip{\hat \nu_s(z) - \hat \theta_s, x_s^\UCB}| + \beta_{s,s^2}^{1/2} \|x_s^\UCB\|_{{V_s}^{-1}}\big)^2 \geq \tfrac{1}{2}\beta_{s,s^2}\|x_s^\UCB\|_{V_s^{-1}}^2\,,
  \end{align*}
  where we used the definition of $q_s$ as a distribution supported on $\xX \setminus \{\hat x_s\}$.
Finally, Lemma \ref{lem:info-bound-agnostic} provides a worst-case bound on the total information gain, $\gamma_n \leq \oO\big(d^2\log(n)^2\big)$,
which is a direct consequence of the elliptic potential bound (Lemma \ref{lem:elliptic-potential}) and the soft-min inequality \eqref{eq:softmin}.
	We conclude $R_n \leq \oO\big(d \sqrt{n}\log(n)\big)$. 
\end{proof}

Our next result is an instance-dependent logarithmic regret bound. The proof follows along the same lines as the worst-case regret bound, but replaces the worst-case bound on the information ratio with an instance-dependent bound. Interestingly, our bound is attained by a distribution with a close resemblance with Thompson sampling. 
\begin{theorem}[Gap-dependent regret]\label{thm:regret-log} 
	The regret of Algorithm \ref{alg:ids-anytime} is bounded by
	\begin{align*}
		R_n \leq \oO\left(\Delta_{\min}^{-1} d^3\log(n)^2 \right) \,.
	\end{align*}
\end{theorem}
Besides universal constants, the $\oO$-notation in the theorem statement hides only the constants required for boundedness of $\xX$ and $\mM$. The proof makes use of the following lemma, which shows an instance-dependent bound on the information ratio. Recall that $\delta_s = \hat \Delta_s(\hat x_s)$ is the gap estimate of the empirically best action, and $\hat \Delta_s(x) = \delta_s + \ip{\hat x_s - x, \hat \theta_s}$.
\begin{lemma}\label{lem:ratio-bound-instance}
	At any local time $s$ with $\beta_{s,s^2} \geq \beta_s \deq \|\hat \theta_s - \thetaopt\|_{V_s}^2$, the optimal information ratio is bounded as follows,
	\begin{align*}
		\min_{\mu \in \sP(\xX)}	\Psi(\mu) \leq \frac{4 \delta_s (8d + 9) }{\Delta_{\min}} \,.
	\end{align*}
\end{lemma}
\begin{proof}
	Let $a \geq 2$ be a constant to be chosen later. If $2 a \delta_s \geq \Delta_{\min}$, then $\min_{\mu \in \sP(\xX)} \Psi_s(\mu) \leq \frac{4 a \delta_s}{\Delta_{\min}}$ by \eqref{eq:ratio-worst-case}. Hence we may assume $2 a \delta_s \leq \Delta_{\min}$ in the following. By (\ref{eq:gap-bound}), for all $s$ 
  with $\beta_s \leq \beta_{s,s^2}$ and $x \neq x^*$, it holds that $\Delta_{\min} \leq 2 \hat \Delta_s(x)$, so in particular $\hat x_s = x^*$. 
	Define $\tilde \mu = \frac{1}{2} e_{\hat x_s} + \frac{1}{2} q_s$ to be the uniform mixture\footnote{By a concentration of measure argument (Appendix \ref{app:bayesian}), the weights $q_s(x)$ approximate the posterior probability of an action $x$ being preferred over $\hat x_s$ by the Bayesian model with Gaussian prior and likelihood. As such, the distribution $\tilde \mu$ resembles the \emph{top-two Thompson sampling} approach proposed by \citet{russo2020simple}.} of $q_s$ and a Dirac at $\hat x_s$.
  Let $\bar \Delta_s(x) = \ip{\hat \theta_s, \hat x_s - x}$ and 
  note that $\bar{\Delta}(\tilde \mu) \geq (a-1) \delta_s \geq \delta_s$ by the assumption $a \geq 2$. Therefore, by Lemma~\ref{lem:two-action-ids}, 
	\begin{align}
		\Psi_s(\mu_s) \leq \min_{p \in [0,1]} \frac{(1-p) \delta_s + p \hat \Delta_s(\tilde\mu)}{p I_s(\tilde \mu)} \leq \frac{4 \delta \bar \Delta_s(\tilde \mu)}{I_s(\tilde \mu)} \,.
    \label{eq:tai}
	\end{align}
	Note that we can bound the information gain $I_s(\tilde \mu)$ as follows,
	\begin{align*}
		I_s(\tilde \mu) \geq \sdfrac{1}{2}\sum_{x \in \xX} \tilde \mu(x) \sum_{z \neq \hat x_s} q_s(z) \ip{\hat \nu_s(z) - \hat \theta_s, x}^2 = \tfrac{1}{2} \sum_{z \neq \hat x_s} q_s(z) \min_{\nu \in \hh{\hat x_s}{z}} \|\nu - \hat \theta_s\|_{V(\tilde \mu)}^2\,.
	\end{align*}
	On the other hand, we can bound the gap $\bar\Delta_s(x) = \ip{\hat \theta_s, \hat x_s - x}$,
	\begin{align*}
		\ip{\hat \theta_s, \hat x_s - x} 
    = \min_{\nu : \ip{\nu, x - \hat x_s} \geq 0} \|\nu - \hat \theta_s\|_{V(\tilde \mu)} \|\hat x_s - x\|_{V(\tilde\mu)^{-1}} \leq \min_{\nu \in \hh{\hat x_s}{x}} \|\nu - \hat \theta_s\|_{V(\tilde \mu)} \|\hat x_s - x\|_{V(\tilde\mu)^{-1}}\,.
	\end{align*}
  Combining the last two displays with the definition of $\tilde \mu$, the fact that $\hat x_s = x^*$ and Cauchy-Schwarz,
	\begin{align*}
		\bar \Delta_s(\tilde \mu)^2 &\leq \sdfrac{1}{4} \sum_{x \neq \hat x} q_s(x) \min_{\nu \in \cC_x}\|\nu - \hat \theta_s\|_{V(\tilde\mu)}^2 \sum_{x \neq \hat x} q_s(x) \|\hat x_s - x\|_{V(\tilde \mu)^{-1}}^2\\
		&\leq (1 + d) \sum_{x \neq \hat x} q_s(x) \min_{\nu \in \hh{\hat x_s}{x}}\|\nu - \hat \theta_s\|_{V(\tilde\mu)}^2 \leq 2(1+d) I_s(\tilde \mu)\,.
	\end{align*}
The second last step uses $\sum_{x \neq \hat x_s} q_s(x) \|x\|_{V(\tilde \mu)^{-1}}^2 \leq 2 \sum_{x \neq \hat x_s} q_s(x) \|x\|_{V(q_s)^{-1}}^2 = 2d$ and $\|\hat x_s \|_{V(\tilde \mu)^{-1}}^2 \leq 2$. 
Next, for $x \neq \hat x_s$, 
\begin{align*}
\bar \Delta_s(x) = \hat \Delta_s(x) - \delta_s \geq \frac{1}{2} \Delta_{\min} - \delta_s \geq \frac{1}{2}\left(1 - \frac{1}{a}\right) \Delta_{\min}\,. 
\end{align*}
Hence, by the definition of $\tilde \mu$ we have $\bar \Delta_s(\tilde{\mu}) \geq \frac{1}{4}(1-1/a) \Delta_{\min}$ and
using (\ref{eq:tai}),
	\begin{align*}
		\Psi_s(\mu_s) &\leq \frac{4 \delta_s \bar \Delta_s(\tilde \mu)}{I_s(\tilde \mu)} = \frac{4 \delta_s \bar \Delta_s(\tilde \mu)^2}{\bar \Delta_s(\tilde \mu) I_s(\tilde \mu)} \leq \frac{32 \delta_s (1+ d)}{\Delta_{\min}\left(1 - \frac{1}{a}\right)} \,.
	\end{align*}
The claim follows with $a=8(1+d)+1$.
\end{proof}

\begin{proofof}{Theorem~\ref{thm:regret-log}}
	Recall that $B_s = \chf{\beta_s \leq \beta_{s, s^2}}$ with $\beta_s = \|\hat \theta_s - \thetaopt\|_{V_s}^2$.	
  As before, by Lemma~\ref{lem:exploration-regret} and using that $\Delta(x_s) B_s\leq  2 \hat \Delta_s (x_s) B_s$, 
	\begin{align*}
		R_n \leq 2 \EE\left[\sum_{s=1}^{s_n}\hat \Delta_s (x_s) B_s\right] + \oO(\log \log(n))\,.
	\end{align*}
	Let $\gamma_n = \sum_{s=1}^{s_n} I_s(x_s)$ be the cumulative information gain. Using Cauchy-Schwarz 
  and the instance-dependent bound on the information ratio from Lemma~\ref{lem:ratio-bound-instance}, 
	\begin{align*}
		\EE\left[\sum_{s=1}^{s_n}\hat \Delta_s (x_s) B_s\right]^2 &\leq \EE\left[\sum_{s=1}^{s_n}\Psi_s(\mu_s) B_s\right]\EE\left[\sum_{s=1}^{s_n}I_s(x_s) B_s\right] \leq \EE\left[\sum_{s=1}^{s_n}\oO\left(\frac{\delta_s d B_s}{\Delta_{\min}}\right)\right]\EE\left[\gamma_n\right]\,.
	\end{align*}
	Further bounding $\delta_s \leq \hat \Delta_s(x_s)$ on the right-hand side and re-arranging yields
	\begin{align*}
		\EE\left[\sum_{s=1}^{s_n}\hat \Delta_s (x_s) B_s\right] \leq \oO\left({d}{\Delta_{\min}^{-1}}\right)\EE\left[\gamma_n\right]\,.
	\end{align*}
	The worst-case total information gain is at most $\EE[\gamma_n] \leq \oO\big(d^2\log(n)^2\big)$ according to Lemma~\ref{lem:info-bound-agnostic}, and the claim follows. We remark that the bound can be improved by bounding the term $\EE[\log(s_n)]$ (which appears in the upper bound on $\gamma_n$) more carefully with the help of Lemma \ref{lem:effective-horizon}.
\end{proofof}

Our next result shows that the proposed version of IDS is asymptotically optimal.
The key insight is a connection between information-directed sampling and a primal-dual approach based on online learning 
to solve the convex program that defines the lower bound. Conceptually, the connection is explained best with an \emph{oracle analysis}, which sets aside the statistical estimation process and highlights the key steps (Appendix~\ref{app:oracle}). In particular, Lemma \ref{lem:ratio-bound-asymptotic} shows that in the asymptotic regime, the information ratio satisfies $\Psi_s(\mu_s) \leq 4 \delta_s \big(c^*  + \oO(\beta_s^{1/2}m_s^{-1/2} + \delta_s )\big)$. Further, Lemma~\ref{lem:asymptotic-information} improves the bound on the total information gain to $\gamma_n = \sum_{s=1}^{s_n} I_s(x_s) \leq \log(n) + o(\log(n))$. Lastly, Lemma~\ref{lem:effective-horizon} shows that IDS samples informative actions with large enough probability that $\EE[\log(s_n)] \leq \oO(\log \log(n))$, which is important to bound lower-order terms in our analysis. 

\begin{theorem}[Asymptotic regret]\label{thm:regret-asymptotic} Algorithm \ref{alg:ids-anytime} is asymptotically optimal, 
	\begin{align*}
		\lim_{n \rightarrow \infty} \frac{R_n}{\log(n)} = c^*\,,
	\end{align*}
	where $c^*$ is the solution to the lower bound \eqref{eq:lower-linear} and we assume that $\|x^*\| > 0$.
\end{theorem}
We sketch the proof below and defer the complete proof to Appendix \ref{app:asymptotic}. The assumption $\|x^*\| > 0$ is used in Lemma \ref{lem:effective-horizon} to show that there are not too many exploration steps, which follows from lower bounding the exploration probability. On the other hand, when $\|x^*\| = 0$, the geometry of the lower bound changes, because the optimal action provides no information. Whether the assumption is necessary for Algorithm \ref{alg:ids-anytime} remains to be determined. As a remedy, we can also replace the gap estimates with thresholded gaps $\hat \Delta_s^{+}(x) = \ip{\hat x_s - x, \hat \theta_s} + \delta_s^+$, where $\delta_s^+ = \max(\delta_s, 1/\sqrt{s})$. Lower bounding the gaps this way ensures that an exploratory action is sampled with probability at least $1/\sqrt{s}$ in each exploration round. We believe that with a thresholded gap estimate, the statement of Theorem \ref{thm:regret-asymptotic} holds without restrictions and Theorems \ref{thm:regret-worstcase} and \ref{thm:regret-log} remain valid. Since it is unclear if the assumptions is required and for simplicity of the proofs, we work with the assumption $\|x^*\| > 0$.

\begin{proof}\textit{(Sketch)} The first step is to improve the bound on the information ratio in the asymptotic regime.	Recall that $\beta_{s} = \|\hat \theta_s - \thetaopt \|_{V_s}^2$ and $m_s = \frac{1}{2}\min_{z \neq x^*} \|\hat \nu_s(z) - \hat \theta_s\|_{V_s}^2$. Then by Lemma \ref{lem:ratio-bound-asymptotic},
	\begin{align*}
		\Psi_s(\mu_s) \leq 4 \delta_s \big(c^*  + \oO(\beta_s^{1/2}m_s^{-1/2} + \delta_s )\big) \,,
	\end{align*}
	for $\beta_s^{1/2}m_s^{-1/2} \rightarrow 0$ and $\delta_s \rightarrow 0$.
Not surprisingly, the proof bounds the information ratio  using a sampling distribution informed from the lower bound \eqref{eq:lower-linear}. Details are given in Appendix \ref{app:information-ratio}. 

Second, we improve the bound on the total information gain $\gamma_n = \sum_{s=1}^{s_n} I_s(x_s)$. The key insight is to interpret the information gain as the loss of an online learning algorithm. We adapt the standard regret proof for the exponential weights algorithm \cite{orabona2019modern}, to bound the total information gain relative to the minimum constraint (Lemma \ref{lem:asymptotic-information}). Informally, the result states that
\begin{align*}
	\EE[\gamma_n] \leq \EE\left[\min_{x \neq \hat x_{s_n}} \|\hat \nu_{s_n}(x) - \hat \theta_{s_n}\|_{V_{s_n}}^2 + \oO(\log(n)^{1/2}\log(s_n))\right] \,.
\end{align*}
Exploration rounds are defined by condition \eqref{eq:exploitation} to ensure that the minimum remains small,
\begin{align*}
	\min_{z \neq \hat x_{s_n}} \|\hat \nu_{s_n}(z) - \hat \theta_{s_n}\|_{V_{s_n}}^2 \leq \beta_{s_n, n\log n} \leq 2 \log(n \log(n)) + \oO(d \log(s_n))\,.
\end{align*}
This result improves upon the worst-case bound on the information gain (Lemma \ref{lem:info-bound-agnostic}), as long as the number of exploration rounds $s_n$ is not too large. 

Third, the proof hinges on Lemma \ref{lem:effective-horizon}, which shows that $\EE[\log(s_n)] \leq \oO(\log \log(n))$. Intuitively, IDS samples informative actions with large enough probability to ensure that in expectation, there is only a logarithmic number of exploration rounds, while the exploration probability is small enough to bound the worst-case regret.

In the remaining proof sketch we only discuss the case where $\Psi_s(\mu_s) \leq 4 \delta_s (c^* + o(1))$ and $\EE[\gamma_n] \leq \log(n) + o(\log(n))$ holds  (the actual proof requires to also bound the regret in early rounds, when the asymptotic statements do not hold). Asymptotically, the mean gap estimate $\bar \Delta_s(x) \deq \ip{x_s - x, \hat \theta_s}$ is a good estimate of the actual regret. Therefore, we get
\begin{align*}
	R_n \leq \EE\bigg[\sum_{t=1}^{s_n} \bar \Delta_s(x_s)\bigg] + o(\log(n))
\end{align*}
Next using that $4ab \leq (a + b)^2$ and Cauchy-Schwarz combined with a few applications of the tower rule, we get
\begin{align*}
	\EE\bigg[\sum_{s=1}^{s_n} \bar \Delta_s(x_s)\bigg]=	\EE\bigg[\sum_{s=1}^{s_n} \bar \Delta_s(\mu_s)\bigg] 
	&\leq \frac{1}{4} \EE\bigg[\sum_{s=1}^{s_n}\delta_s\bigg]^{-1} \EE\bigg[\sum_{s=1}^{s_n} \hat \Delta_{s }(\mu_s)\bigg]^2 \nonumber \\
	&\leq \frac{1}{4} \EE\bigg[\sum_{s=1}^{s_n}\delta_s\bigg]^{-1} \EE\bigg[\sum_{s=1}^{s_n} \Psi_{s}(\mu_s)\bigg]\EE\bigg[\sum_{s=1}^{s_n} I_s(x_s)\bigg]\,.
\end{align*}
The bound on the information ratio yields
\begin{align*}
	\frac{1}{4} \EE\bigg[\sum_{s=1}^{s_n}\delta_s\bigg]^{-1} \EE\bigg[\sum_{s=1}^{s_n} \Psi_{s}(\mu_s)\bigg] \leq \frac{1}{4} \EE\bigg[\sum_{s=1}^{s_n}\delta_s\bigg]^{-1} \EE\bigg[\sum_{s=1}^{s_n} 4 \delta_s (c^* + o(1))\bigg] \leq c^* + o(1)\,.
\end{align*}
Combined with the bound on the information gain, asymptotic optimality follows.
\end{proof}

\subsection{Alternative Definitions of the Information Gain}

Our definition of the information gain ensures that $I_{s}(x) \approx \tfrac{1}{2}\sum_{z \neq \hat{x}_{s}} q_s(z) \ip{\hat \nu_{s}(z) - \htheta_{s},x}^2$ asymptotically.
In finite time, however, the mean estimates can be inaccurate. Therefore, we add an optimistic term in the definition of the information gain \eqref{eq:I-def}, which is an essential ingredient in the proof of Theorem \ref{thm:regret-worstcase}. At the same time, the optimistic term corresponds to an information gain which was analyzed in earlier work \citep{Kirschner2018,kirschner20partialmonitoring}. Since this choice is motivated from a worst-case perspective, empirically it sometimes leads to over-exploration in the finite-time regime. 
A closer inspection of the worst-case regret proof (in particular, Eq.~\ref{eq:ratio-worst-case}) reveals that the optimistic term is only needed for the UCB action. 
This motivates the following definition: \looseness=-1
\begin{align}
	I_{s}^{\IAUCB}(x) = \sdfrac{1}{2}\sum_{z \neq \hat{x}_{s}} q_s(z) \left(|\ip{\hat \nu_{s}(z) - \htheta_{s},x}| + \chf{x = x_s^{\UCB}} \beta_{s,s^2}^{1/2}\|x\|_{V_s^{-1}}\right)^2\,. \label{eq:info-gain-ucb}
\end{align}
With a few additional steps in the proof of Lemma \ref{lem:ratio-bound-asymptotic} and Theorem \ref{thm:regret-asymptotic}, the resulting algorithm is shown to satisfy the same regret bounds as presented in Theorems \ref{thm:regret-worstcase}, \ref{thm:regret-log} and \ref{thm:regret-asymptotic}. Since the proofs are very similar, we omit the details. We compare both information gain functions in our experiments.
Another variant is to set the alternative parameters to
\begin{align*}
	\tilde \nu_s(x) = \argmin_{\nu \in \cC_x} \|\nu - \hat \theta_s\|_{V_s}^2\,,\quad \text{where}\quad \cC_x = \{\nu \in \mM : \max_{z \in \xX} \ip{\nu, z - x} = 0 \}\,.
\end{align*}
Note that $\cC_x$ is the set of parameters where $x$ is optimal and is sometimes called the \emph{cell} of $x$. Let $\tilde q(z) \propto \exp(-\eta \|\tilde \nu_s(z) - \hat \theta_s\|_{V_s}^2)$ and define 
\begin{align}
	I_{s}^\IACELL(x) \deq \sdfrac{1}{2}\sum_{z \neq \hat{x}_{s}} \tilde q_s(z) \left(|\ip{\tilde \nu_{s}(z) - \htheta_{s},x}| + \beta_{s,s^2}^{1/2}\|x\|_{V_s^{-1}}\right)^2\,. \label{eq:info-gain-cell}
\end{align}
Note that all bounds that we obtain hold true for IDS defined with $I_s^\IACELL$ as well, by replacing $\hh{\hat x_s}{x}$ with $\cC_x$ in the proof. The key insight is that $\cC^* = \cup_{x \neq x^*} \cC_x = \cup_{x \neq x^*}\hh{x^*}{x}$, hence the change is simply a different decomposition of the set of alternative parameters $\cC^*$ into convex regions. One might expect faster convergence from the fact that $\tilde q_s$ is more concentrated, but empirically we find little difference compared to $I_s$. On the other hand, for unconstrained parameter sets $\mM$, we can compute $\hat \nu_s(z)$ in closed form (Eq.~\ref{eq:nu-closed-form}), whereas $\tilde \nu_s(z)$ can only be computed by solving a positive definite quadratic program with $k$ linear constraints for each action $z \neq \hat x_s$. Interestingly, however, the information gain \eqref{eq:info-gain-cell} relates to the Bayesian mutual information $\II_s(y_s;x^*|x_s=x)$. The argument uses concentration of measure to show that $\tilde q_s(x)$ approximates the posterior probability that an action $x \neq x^*$ is optimal in the Bayesian model. We refer to Appendix \ref{app:bayesian} for details. 

\subsection{Information-Directed Sampling as a Primal-Dual Approach}\label{ss:primal-dual}

Lemma \ref{lem:ids-support} shows that the IDS distribution $\mu_s$ is supported on actions $x$ that minimize the function
\begin{align*}
g_s(x) = \hDelta_s(x)  - \frac{\Psi_s(\mu_s)}{2 \hDelta_s(\mu_s)} I_s(x) \stackrel{n \rightarrow \infty}{\approx} \hDelta_s(x)  - c^*I_s(x) \,.
\end{align*}
The approximation holds because asymptotically, $\Psi_s(\mu_s) \approx 4c^* \delta_s$ and $\hat \Delta_s(\mu_s) \approx 2 \delta_s$. The weight $c^*$ appears from normalizing the Lagrange multipliers as discussed in Appendix~\ref{app:oracle}. Therefore, the IDS distribution can be understood as a type of best-response on the primal-dual game defined by the Lagrangian of the lower bound, where the dual variables correspond to the $q$-weights of the information gain. Note that the best response on $g_s$ is not unique, and IDS chooses a particular, randomized trade-off, which is imposed by the IDS objective \eqref{eq:ids-def}.

The first work which exploits the primal-dual formulation for regret minimization is by \cite{degenne2020structure}. In our notation, their algorithm corresponds to choosing the action with the best information-regret trade-off $z_s= \argmin_{x \in \xX} {\hat \Delta_s(x)}/{I_s(x)}$. IDS instead asymptotically randomizes between $x^*$ and $z_s$, which allows it to maintain the worst-case regret bound. Another more recent primal-dual approach is the \solid{} algorithm by \citet{tirinzoni2020asymptotically}. This approach uses a different Lagrangian, which is defined by keeping the minimum over $\cC^*$ in \eqref{eq:lower-linear}. Accordingly, the dual variable is one-dimensional, but the constraints appear non-smooth. \solid{} is defined by alternating (optimistic) sub-gradient steps on the allocation and the dual variable. This leads to a randomized strategy over actions with exponential weights that are only updated when an exploration condition is satisfied. 

%% file: parts/ids-experiments.tex
\section{Experiments}
\label{sec:experiments}

\begin{figure}
    \centering
    \includegraphics{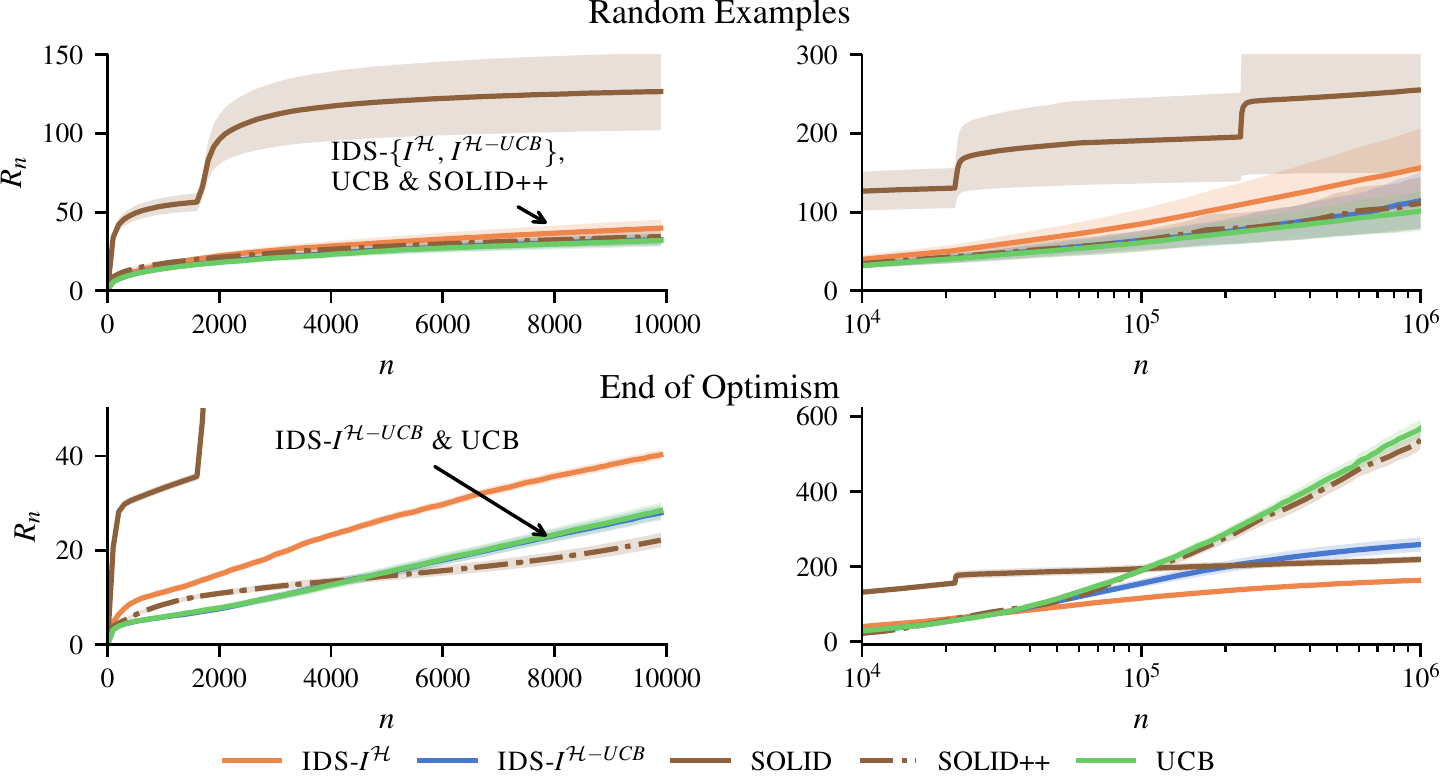}
    \caption{Top: Worst-case regret on randomly drawn action sets. Bottom: Counter-example problem from \cite{lattimore2017end}. Early stages are shown in linear scale, asymptotics in log scale. Results are averaged over 100 repetitions and the confidence region shows $2 \times$standard error.}
    \label{fig:all-plots}
\end{figure}


We compare \IDS{} with \LinUCB{} \citep{Abbasi2011improved} and \solid~\citep{tirinzoni2020asymptotically}, the latter being our closest competitor. Note that \solid{} was shown to outperform OAM \citep{hao2019adaptive} and \LinTS{} in a variety of settings. To the best of our knowledge, \solid{} is the current state-of-the-art for asymptotically optimal algorithms.

To enable a fair comparison, we use the same confidence coefficient $\beta_{t,1/\delta}$ \eqref{eq:confidence-coefficient} for all algorithms. We also run the same experiment with the (tighter) confidence coefficient derived by \cite{tirinzoni2020asymptotically}, but we found no significant difference in the results, see Appendix~\ref{ap:experiments-supp}. For \solid, we use the default hyper-parameters suggested by \citet[Appendix K]{tirinzoni2020asymptotically}. Finally, as recommended by the authors, we implement a variant of the \solid{} algorithm, which is (heuristically) optimized for better performance in finite time and does not reset the sampling vector $\omega_t$ at the beginning of each phase. We display that improved version as \solid++. 

\IDS{} is implemented as in Algorithm~\ref{alg:ids-anytime} with the computational improvements described at the end of Section \ref{sec:ids}. In particular, we use an unconstrained parameter set ($\mM=\RR^d$), which allows us to compute the parameter $\hat \nu_s(x)$ in closed form. We further compute the IDS distribution randomizing only between $\hat x_t$ and one other action (Lemma~\ref{lem:ids-support}) to reduce the per-round computational complexity from $\oO(k^2)$ to $\oO(k)$. All variants of IDS used in the experiments satisfy the theoretical guarantees presented in this paper with minor proof modifications.  We also compare to \IDS-$I^\IAUCB$ defined with information gain \eqref{eq:info-gain-ucb}. In Appendix \ref{ap:experiments-supp}, we present further empirical evidence, including a benchmark with Thompson Sampling and Bayesian IDS, a comparison of information gain functions, and an evaluation of the tuning sensitivity of the $\beta_s$ and $\eta_s$ parameters.

\paragraph{Average performance on random problems.} For each repetition, we sample an action set with 6 actions drawn uniformly from the unit sphere. We set $d=2$ and the variance of the noise to $\sigma^2=0.1$, which is chosen so that the asymptotic regime is observed after fewer rounds relative to $\sigma^2 = 1$. The results are shown in the first row of Figure~\ref{fig:all-plots}. We display the average over 100 runs and $95\%$ confidence intervals. All policies except for \solid{} have comparable averaged performances, but the latter is not designed to optimize for worst-case regret in principle. \IDS-$I^\IAUCB$ is similar to \LinUCB{}, followed by \IDS-$I^\IA$.

\paragraph{The End of Optimism?} 
\begin{wrapfigure}{r}{0.5\textwidth}
	\begin{center}
		\includegraphics[scale=0.9]{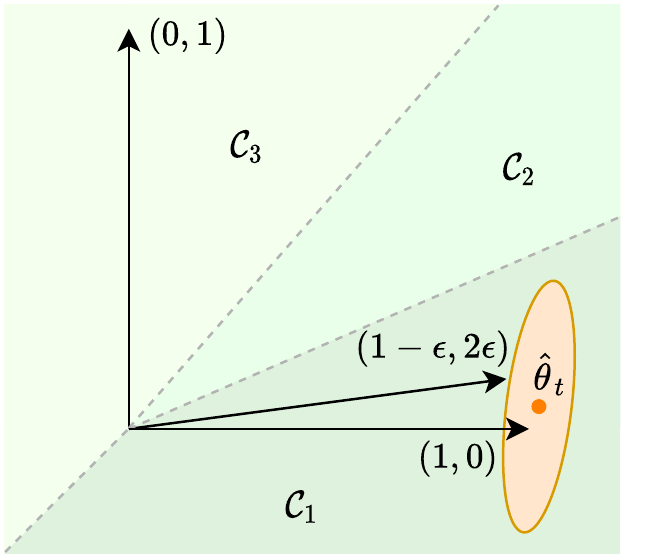}
	\end{center}
	\caption{The `end of optimism' example.}\label{fig:eoo-example}
\end{wrapfigure}
This example of a 2-dimensional linear bandit dates back to \citet[Appendix A]{soare2014best}, and was used by \citet{lattimore2017end} to show that algorithms based on optimism and Thompson sampling are not asymptotically optimal in the linear setting. 
There are three arms $x_1 = (1, 0)$, $x_2 = (1-\epsilon, 2\epsilon)$ and $x_3 = (0,1)$ with a tuning variable $\epsilon > 0$. The true parameter is $\theta = (1,0)$ which makes action $x_1$ optimal. 
The situation is illustrated in Figure \ref{fig:eoo-example}. The colored regions $\cC_1, \cC_2$ and $\cC_3$ are the corresponding \emph{cells}, i.e.~the subset of parameters in $\RR^2$ for which $x_1$, $x_2$ or $x_3$ is optimal respectively. 
When the confidence ellipsoid $\eE_t = \{ \theta : \|\theta - \htheta_t\|_{V_t}^2 \leq c\log(n)\}$ for the least squares estimator $\hat \theta_t$ is contained in the cell $\cC_1$, the learner has identified the best action with high probability.

Algorithms based on optimism and Thompson sampling quickly rule out the suboptimal arm $x_3$ and just play either $x_1$ or $x_2$. The twist is that the third arm is still informative for determining $a^*$, and in fact an asymptotically optimal algorithm plays only on $\{x_1, x_3\}$. To see why, note that any no-regret learner plays $x^*=x_1$ a lot, therefore the parameter is well-estimated along the direction $x_1$. 
It remains to shrink the confidence ellipsoid approximately along the direction $x_3$, which means increasing the $V_t$-norm of $x_3$. Choosing arm $x_2$ incurs a small cost $\epsilon$, but the increase of the confidence ellipsoid in direction $x_3$ is only small, $\ip{x_3, (V_{t+1} - V_t) x_3} = \ip{x_3, x_2}^2 = \epsilon^2$.
On the other hand, choosing $x_3$ implies a higher regret cost of  $1$, but the confidence set is increased by $1$ along direction $x_3$, which allows to identify the optimal action at a much smaller cost. An optimistic algorithm has an asymptotic regret that scales with $R_n \approx \log(n)/\epsilon$, while for an optimal algorithm, $R_n \approx 1 \cdot \log(n)$. In fact, for some small $\epsilon$, the lower bound constant \eqref{eq:lower-linear} is $c^* = 64$ and does not depend on $\epsilon$, so optimistic algorithms cannot be asymptotically optimal.


For the experiments, we use noise variance $\sigma^2=0.1$, and $\epsilon = 0.01$, which is sufficiently large to reach the asymptotic regime within $n=10^6$ rounds, and small enough to highlight the difference between UCB and IDS.
%
%
Results in this setting are shown in the bottom row of Figure~\ref{fig:all-plots}. As expected, \LinUCB's asymptotics show a suboptimal log-slope, but it is surprisingly followed by \solid++. Despite our attempts, we are presently not able to provide a good explanation for this result and it might require a more involved analysis of the \solid++ heuristic. However, both versions of \IDS{} and the theoretical \solid{} reach the optimal asymptotic around $t=10^5$ ($10^4$ for \solid{}) and significantly outperform \LinUCB{} on that problem. An interesting observation is that IDS-$I_s^{\UCB}$ performs better in finite time, whereas IDS-$I_s$ reaches the asymptotic regime earlier.


%% file: parts/conclusion.tex
\section{Conclusion}\label{sec:conclusion}

We introduced a simple and efficient algorithm for linear bandits that is (nearly) worst-case optimal and matches the asymptotic lower bound exactly. Note that the algorithm is essentially hyper-parameter free with the usual boundedness assumptions. Nonetheless, the confidence parameter $\beta_{s,1/\delta}$ and the learning rate $\eta_s$ used in the definition of $I_s$ provide some tuning knobs to improve performance in practice. 

Our theoretical results still rely on some restrictive assumptions, such as the boundedness requirement for the parameter set, uniqueness of $x^*$ and $\|x^*\| > 0$ for the asymptotic regret, and the need to discard data in exploitation rounds. Also, the dependence on $d$ and $k$ is sub-optimal in some regimes, in particular for the worst-case regret bound and small $k$. On the upside, our analysis is relatively simple, and raises the hope that there exists a \emph{really} simple proof. Finding an information gain which preserves the guarantees and telescopes more easily could be a first step towards this end. \looseness=-1

Finally, it appears likely that our framework generalizes in several directions. The contextual case is already covered in previous work on asymptotic algorithms \citep{hao2019adaptive,tirinzoni2020asymptotically}. We point out that IDS can be defined to optimize the marginals of the joint distribution between context and action \citep{kirschner20partialmonitoring}. Decoupling the reward from the observation features leads to the linear partial monitoring framework, where IDS is known to achieve the optimal worst-case rate in all possible games \citep{kirschner20partialmonitoring}. The structured bandit setting and information gain functions for a non-Gaussian likelihood are yet other promising directions.

%% file: parts/notation.tex
\section{Notation}\label{app:notation}
	\begin{center}
			\setlength{\tabcolsep}{4pt}
		\begin{longtable}{c c p{12cm} }

			\multicolumn{3}{l}{\textit{Linear Bandit Setting}}\\[2pt]
			$d$ &  & feature dimension\\
			$\xX$ & $\subset$ & $\RR^d$, action (feature) set\\
			$\mM$ & $\subset$ & $\RR^d$, parameter set\\
			$\theta^*$ & $\in$ & $\mM$, unknown, true parameter\\
			$k$ & $\triangleq$ & $|\xX|$, number of actions\\
			$x^*$ & $\triangleq$ & $\argmax_{x \in \xX} \ip{x, \theta^*}$, best action\\
			$\Delta(x)$ & $\triangleq$ & $\ip{x^* - x, \theta^*}$, suboptimality gap\\
			$\cC_x$ & $\triangleq$ & $\{ \nu \in \mM : \ip{x, \nu} \geq \max_{y \in \xX} \ip{y, \nu}\}$, cell of action $x$\\
			$\hh{x^*}{x}$ & $\triangleq$ & $\{ \nu \in \mM : \ip{x, \nu} \geq \ip{x^*, \nu}\}$\\
			$n$ & $\triangleq$ & horizon\\
			$R_n$ & $\triangleq$ & $\sum_{t=1}^n \ip{x^* - x_t, \thetaopt}$, regret\\
			$s_n$ & $\triangleq$ & effective horizon / exploration step counter\\
			$c^*$ & $\triangleq$ & asymptotic regret, see \eqref{eq:lower-linear}\\
			$\alpha^*$ & $\triangleq$ & asymptotically optimal allocation\\
			$x_s$ & $\triangleq$ & $x_{t_s}$ action choice at local time $s$\\
			$y_s$ & $\triangleq$ & $\ip{x_s, \thetaopt} + \epsilon_{s}$, observation with (sub-)Gaussian noise $\epsilon_s$\\
			\multicolumn{3}{l}{}\\[-5pt]
			\multicolumn{3}{l}{\textit{Least-Squares Estimate}}\\[2pt]
			$V(\alpha)$ & $\triangleq$ & $\sum_{x \in \xX} \alpha(x) x x^\T$, covariance matrix for allocation $\alpha$\\
			$V_s$ & $\triangleq$ &  $\sum_{i=1}^s x_i x_i^\T + \textbf{1}_d$, (regularized) empirical covariance matrix \\
			$\hat \theta_s$ & $\triangleq$ & $V_s^{-1}\sum_{i=1}^s x_i y_i$, least squares estimate\\
			$\beta_{s, 1/\delta}$ & $\triangleq$ & $\left(\sqrt{2 \log \frac{1}{\delta} + \log \det V_s} + 1\right)^2$ concentration coefficient\\
			$\hat x_s$ & $\triangleq$ & $\argmax_{x \in \xX} \ip{x, \hat \theta_s}$, empirically best action for the estimate $\hat \theta_s$\\
			$x_s^\UCB$ & $\triangleq$ & $\argmax_{x \in \xX} \ip{x, \hat \theta_s} + \beta_{s,s^2}^{1/2}\|x\|_{V_s^{-1}}$, UCB action\\
			$\hat \nu_s(x)$ & $\triangleq$ & $\argmin_{\nu \in \cC_x} \|\nu - \hat \theta_s\|_{V_s}^2$, alternative parameter in $\cC_x$\\
			$m_s$ & $\triangleq$ & $\frac{1}{2}\min_{x \neq \hat x_s} \|\hat \nu_s(x) - \hat \theta_s \|_{V_s}^2$, minimum constraint value\\
			\multicolumn{3}{l}{}\\[-5pt] 
			\multicolumn{3}{l}{\textit{Information-Directed Sampling}}\\[2pt]
			$\hat \Delta_s(x)$ & $\triangleq$ & $\delta_s + \ip{\hat x_s - x, \hat \theta_s}$ gap estimate with estimation error $\delta_s$\\
			$I_s(x)$ & $\triangleq$ & information gain\\
			$\gamma_n$ & $\triangleq$ & $\sum_{s=1}^{s_n} I_s(x_s)$, total information gain\\
			$\Psi_s(\mu)$ & $\triangleq$ & $\frac{\hat \Delta_s (\mu)^2}{I_s(\mu)}$, information ratio\\
			$\mu_s$ & $\triangleq$ & $\argmin_{\mu \in \sP(\xX)} \Psi_s(\mu)$, IDS distribution\\
		\end{longtable}
	\end{center}
\newpage

%% file: parts/appendix-ids.tex
\section{Additional Proofs and Technical Lemmas}\label{app:proofs}

\subsection{Properties of the IDS Distribution}
The results in this section are mostly known or refine previous results. We start with a lemma by \citet[Lemma 5]{kirschner20partialmonitoring}, which shows that IDS plays close to greedy.
\begin{lemma}[Almost greedy]\label{lem:almost-greedy}\label{lem:greedy-probabilistic} The IDS distribution is almost greedy, $\hat{\Delta}_s(\mu_s) \leq 2 \delta_s$.
\end{lemma}
The next result is by \citet[Proposition 6]{Russo2014learning}. 
\begin{lemma}[Convexity \& support on two actions] \label{lem:ids-distr}
	The information ratio as a function of the distribution, $\mu \mapsto \Psi_s(\mu)$ is convex. Further, the IDS distribution $\mu_s = \argmin_{\mu} \Psi_s(\mu)$ can always be chosen with a support of at most two actions.
\end{lemma}
In light of this lemma, the IDS distribution can be understood and computed by optimizing the information ratio between pairs of actions. We provide a closed-form solution for the IDS distribution over two actions in the next lemma.
\newcommand{\Done}{\Delta_1}\newcommand{\Dtwo}{\Delta_2}\newcommand{\Ione}{I_1}\newcommand{\Itwo}{I_2}
\begin{lemma} \label{lem:two-action-ids}
	Let $0 < \Done \leq \Dtwo$ denote the gaps of two actions and $0 \leq \Ione, \Itwo$ the corresponding information gain. Define the ratio
	\begin{align*}
		\Psi(p) = \frac{\left((1-p) \Done + p \Dtwo\right)^2}{(1-p)\Ione + p \Itwo}\,.
	\end{align*}
	Then the optimal trade-off probability $p^* = \argmin_{0 \leq p \leq 1} \Psi(p)$ is
	\begin{align*}
		p^* = \begin{cases}
			0 & \text{if } \Ione \geq  \Itwo\\
			\text{\normalfont clip}_{[0,1]}\left(\frac{\Done}{\Dtwo- \Done} - \frac{2\Ione}{\Itwo - \Ione}\right) & \text{else,}
		\end{cases}
	\end{align*}
	where we use the convention that $\Delta_1/0 = \infty$ and $\text{\normalfont clip}_{[0,1]}(a) = \max(\min(a,1),0)$.
\end{lemma}
\begin{proof}
	The case $\Ione \geq \Itwo$ is immediate, because $p > 0$ increases the numerator and decreases the denominator. Hence we can assume $\Ione < \Itwo$. Recall that $\Psi(p)$ is convex on the domain $[0,1]$ (Lemma \ref{lem:ids-distr}). The derivative is
	\begin{align*}
		\Psi'(p) = -\frac{\big(\Done + p(\Dtwo - \Done)\big)\big(\Done(\Itwo- \Ione) - (\Dtwo -\Done) (2 \Ione + p(\Itwo - \Ione))\big)}{(\Ione + p(\Itwo-\Ione))^2}\,.
	\end{align*}
	Note that $(\Dtwo - \Done) \geq 0$ and $(\Itwo - \Ione) > 0$. Solving for the first order condition $\Psi'(p) = 0$ gives $p = \frac{\Done}{\Dtwo- \Done} - \frac{2\Ione}{\Itwo - \Ione}$. We can also read off that $p < 0$ implies $\Psi'(0) > 0$, and $p > 1$ implies $\Psi'(1) < 0$. Hence clipping $p$ to $[0,1]$ leads to the correct solution.
\end{proof}

The next lemma characterizes the support of the IDS distribution. 
\begin{lemma}[IDS support] \label{lem:ids-support}
	Let $\Psi_s^* = \min_\mu \Psi_s(\mu)$ and define
	\begin{align*}
		g_s(x) = \hDelta_s(x)  - \frac{\Psi_s^*}{2 \hDelta_s(\mu_s)} I_s(x)\,.
	\end{align*}
	For any $x \in \supp(\mu_s)$, it holds $g_s(x)= \min_{z \neq \xX} g_s(z)$, and further $g_s(x) = g_s(\mu) = \frac{1}{2} \hat \Delta_s(\mu_s)$.
\end{lemma}
\begin{proof}
	The proof is similar to the proof of \citep[Proposition 6]{Russo2014learning}. It is easy to see that the solution sets to the following objectives are equal:
	\begin{align*}
		\min_{\mu} \frac{\hDelta_s(\mu)^2}{I_s(\mu)} \quad \text{and} \quad \min_{\mu} \left\{ S(\mu) = \hDelta_s(\mu)^2 - \Psi_s^* I_s(\mu) \right\} \,,
	\end{align*}
	where $\Psi_s^* = \min_{\mu} \frac{\hDelta_s(\mu)^2}{I_s(\mu)}$ is the optimal information ratio. Thinking of $\mu_s$ as a vector in $\RR^k$, we compute the gradient of $S(\mu)$ at $\mu_s$,
	\begin{align*}
		\nabla_\mu S(\mu)|_{\mu =\mu_s} = 2 \hDelta_s  \hDelta_s(\mu_s) - \Psi^*_s I_s = h_s \in \RR^k
	\end{align*}
	It must be that for each $x \in \supp(\mu_s)$,  $h_s(x) = \min_x h_s(x)$. Suppose otherwise, that the optimal solution is supported on some $x$ and there exists a $z \neq x$ with $h_s(x) > h_s(z)$. Then $(e_x - e_z)^\T h_s > 0$, hence moving probability mass from $x$ to $z$ would decrease the objective. In other words, the IDS distribution must be minimizing $h_s$,
	\begin{align*}
		h_s(\mu_s) = \min_{\mu} h_s(\mu)\,.
	\end{align*}
	Now, simply dividing $h_s$ by $2\hDelta(\mu_s)$ and taking expectation over the support of the IDS distribution yields the second claim.
\end{proof} 

%% file: parts/appendix-ratio.tex
\subsection{Bounds on the Information Ratio} \label{app:information-ratio}

For the asymptotic bound on the information ratio, we define $\alpha^* \in (\RR_{\geq 0} \cup \{\infty\})^k$ as the solution to the lower bound \eqref{eq:lower-linear}, which is obtained as the appropriate limit. Further, let $\tilde \alpha^* = \alpha^* \chf{x \neq x^*}$ be the optimal allocation on the sub-optimal actions, which is always finite.

\begin{lemma}[Truncated optimal allocation]\label{lem:truncated-optimal}
	Let $\alpha_\lambda^*(x) = \tilde \alpha^* + \lambda \chf{x = x^*}$ be the optimal allocation truncated on $x^*$ such that $\alpha_{\lambda}^*(x^*) = \lambda$. There exists a constant $C(\theta, \xX)$ depending only on the instance and the action set, such that for all $\nu \in \cC^*$,
	\begin{align*}
	\tfrac{1}{2}\|\nu - \thetaopt \|_{V(\alpha_\lambda^*)}^2 \geq 1 - 2 C(\theta, \xX) \|\tilde \alpha\|_1 \lambda^{-1}\,.
	\end{align*}
\end{lemma}
\begin{proof}
	Assume $2 C(\theta, \xX)  \|\tilde \alpha\|_1 \leq \lambda$, otherwise the claim is immediate. Let $\tilde \alpha^*(x) = \alpha^* \chf{x \neq x^*}$ be the optimal allocation on sub-optimal actions. We have
	\begin{align*}
		\tfrac{1}{2}\|\nu - \thetaopt \|_{V(\alpha_\lambda^*)}^2 = \tfrac{1}{2}	\|\nu - \thetaopt \|_{V(\tilde \alpha^*)}^2 + \tfrac{\lambda}{2}	\ip{\nu - \thetaopt, x^*}^2\,.
	\end{align*}
	If $\lambda \ip{\nu - \thetaopt, x^*}^2 \geq 2$ the claim follows. Hence we may assume $\ip{\nu - \thetaopt, x^*}^2 \leq 2\lambda^{-1}$. In other words, $\nu$ is in a $(2/\lambda)^{1/2}$-neighbourhood of the affine subspace, which is defined by $x^*$ and offset $\thetaopt$. Now we fix any $x \neq x^*$, such that $\nu \in \hh{x^*}{x}$ and define $\hH_x^* = \hh{x^*}{x} \cap \{\nu : \ip{\nu - \thetaopt, x^*} = 0\}$ as the intersection of the affine subspace with  $\hh{x^*}{x}$. This is the set of parameters in $\hh{x^*}{x}$, which is indistinguishable from observations of $x^*$. By definition, $\nu^* \in \hH_x^*$ satisfies $\ip{\nu^* - \thetaopt, x^*} = 0$, hence by definition of the optimal allocation,
	\begin{align*}
		\tfrac{1}{2}\|\nu^* - \thetaopt \|_{V(\tilde \alpha^*)}^2 = \tfrac{1}{2} \|\nu^* - \thetaopt \|_{V(\alpha^*)}^2 \geq 1\,.
	\end{align*} 
	We expect the same holds approximately for $\nu$ with $\ip{\nu - \theta^*, x^*}^2 \leq 2\lambda^{-1}$. Lemma \ref{lem:convex-subspace} with an appropriate shift of the parameter space and $\lambda_{\max}(V(\tilde \alpha^*) )  \leq \|\tilde \alpha^*\|_1$ imply
	\begin{align*}
		\min_{\nu^* \in  \hH_x^*} \|\nu - \nu^*\|_{V(\tilde \alpha^*)}^2 \leq C(\theta, \xX) \|\tilde \alpha^*\|_1 \ip{\nu - \thetaopt, x^*}^2 \leq 2\lambda^{-1} C(\theta, \xX)  \|\tilde \alpha\|_1 \leq 1 \,,
	\end{align*}
	where the last inequality is our case assumption. Considering that $\|\nu^* - \thetaopt \|_{V(\alpha^*)} \geq \|\nu - \nu^*\|_{V(\tilde \alpha^*)}$, we further get
	\begin{align*}
		\tfrac{1}{2}\|\nu - \thetaopt \|_{V(\alpha_\lambda^*)}^2& = \tfrac{1}{2}\|\nu - \thetaopt \|_{V(\tilde \alpha^*)}^2 + \tfrac{\lambda}{2}	\ip{\nu - \thetaopt, x^*}^2 \\
		&\geq 	\tfrac{1}{2}(\|\nu^* - \thetaopt \|_{V(\tilde \alpha^*)} - \|\nu - \nu^* \|_{V(\tilde \alpha^*)}) ^2+ \tfrac{\lambda}{2}	\ip{\nu - \thetaopt, x^*}^2\,.
	\end{align*}
	The case $\|\nu^* - \thetaopt \|_{V(\tilde \alpha^*)} \geq 2$ is again immediate, so we may assume $\sqrt{2} \leq \|\nu^* - \thetaopt \|_{V(\tilde \alpha^*)} \leq 2$, which leaves us with
	\begin{align*}
	\tfrac{1}{2}\|\nu - \thetaopt \|_{V(\alpha_\lambda^*)}^2 &\geq \tfrac{1}{2}\|\nu^* - \thetaopt \|_{V(\tilde \alpha^*)}^2 - \|\nu - \nu^* \|_{V(\tilde \alpha^*)}\|\nu^* - \thetaopt \|_{V(\tilde \alpha^*)} + \tfrac{\lambda}{2}	\ip{\nu - \thetaopt, x^*}^2\\
		&\geq 1 - 2  \|\nu - \nu^* \|_{V(\tilde \alpha^*)}  + \tfrac{\lambda}{2}\ip{\nu - \thetaopt, x^*}^2\\
		&\stackrel{(i)}{\geq}1 - 2 (C(\theta, \xX) \|\tilde \alpha\|_1 \ip{\nu - \thetaopt, x^*}^2)^{1/2} + \tfrac{\lambda}{2}	\ip{\nu - \thetaopt, x^*}^2 \\
		&\stackrel{(ii)}{\geq} 1 -  2 C(\theta, \xX) \|\tilde \alpha\|_1 \lambda^{-1} \,.
	\end{align*}
	For (i) we choose $\nu^*$ with $\|\nu - \nu^*\|_{V(\tilde \alpha^*)}^2 \leq C(\theta, \xX) \|\tilde \alpha\|_1 \ip{\nu - \thetaopt, x^*}^2$ and for (ii) we minimize over $\ip{\nu - \thetaopt, x^*}$. This completes the proof.
\end{proof}

\begin{lemma}[Asymptotic information ratio]\label{lem:ratio-bound-asymptotic} 
	Recall that $\beta_{s} = \|\hat \theta_s - \thetaopt \|_{V_s}^2$ and $m_s = \frac{1}{2}\min_{z \neq x^*} \|\hat \nu_s(z) - \hat \theta_s\|_{V_s}^2$. Assume that $4 \beta_s \leq m_s$ and $\beta_s \leq \beta_{s,s^2}$. Then,
	\begin{align*}
		\Psi_s(\mu_s) \leq 4 \delta_s \big(c^*  + \oO(\beta_s^{1/2}m_s^{-1/2} + \delta_s )\big) \,,
	\end{align*}
	for $\beta_s^{1/2}m_s^{-1/2} \rightarrow 0$ and $\delta_s \rightarrow 0$.
\end{lemma}
\begin{proof}
	First note that the assumption $m_s \geq 4 \beta_s$ implies $\hat x_s = x^*$ by Lemma \ref{lem:ms}. Introduce the shorthand $\bar \Delta_s(x) = \ip{\hat \theta_s, \hat x_s - x}$ for the estimated mean gap and let $\tilde \mu \in \sP(\xX)$ be a distribution with $2 \delta_s \leq \hat \Delta(\tilde \mu) = \delta_s + \bar \Delta_s(\tilde \mu)$. Then, by Lemma \ref{lem:two-action-ids},
	\begin{align*}
		\min_{\mu \in \sP(\xX)}\Psi_{s}(\mu) \leq \min_{0 \leq p \leq 1} \frac{\big((1-p)\hat \Delta_s(x^*) + p \hat \Delta(\tilde \mu)\big)^2}{p I_s(\tilde \mu)} = \frac{4 \delta_{s} (\hDelta_{s}(\tilde \mu) - \delta_{s}) }{I_{s}(\tilde \mu)} = \frac{4 \delta_{s} \bar \Delta_{s}(\tilde \mu) }{I_{s}(\tilde \mu)}\,.
	\end{align*}
	Note that the last ratio is invariant in constant rescaling $\tilde \mu$, so we may plug in non-normalized allocations. Recall that $\tilde \alpha^*$ is the optimal allocation over suboptimal actions, as defined at the beginning of Appendix \ref{app:information-ratio}. We let $\alpha_\lambda^* = \tilde \alpha^* + \lambda \chf{x=x^*}$ be the truncated optimal allocation and $\tilde \mu_\lambda = \alpha_\lambda^*/(\|\tilde \alpha^*\|_1 + \lambda)$ be the corresponding normalized distribution. With Lemma \ref{lem:ms} , we get
	\begin{align*}
		\Delta(\tilde \mu_\lambda) - \bar \Delta_s(\tilde \mu_\lambda) \leq  \|\hat \theta_s - \theta^*\|_{V_s} \max_{x \neq x^*} \|x^* -x\|_{V_s^{-1}} \leq \frac{\beta_s^{1/2}}{(2m_s)^{1/2} - \beta_s^{1/2}} \leq \frac{2 \beta_{s}^{1/2}}{m_s^{1/2}}\,.
	\end{align*}
	The last inequality simplifies the expression with $4 \beta_s \leq m_s$. Note that $\Delta(\tilde \mu_\lambda) = \frac{c^*}{\|\tilde \alpha^*\|_1 + \lambda}$. Hence, to satisfy $\delta_s \leq \bar \Delta_s(\tilde \mu_\lambda)$, it is sufficient to satisfy the constraint,
	\begin{align*}
		\delta_s \leq \frac{c^*}{\|\tilde \alpha^*\|_1 + \lambda} - \frac{2\beta_s^{1/2}}{m_s^{1/2}}.
	\end{align*}
	At equality, we get
	\begin{align*}
		\lambda = \frac{c^*}{\delta_s + \frac{2\beta_s^{1/2}}{m_s^{1/2}}} - \|\tilde \alpha^*\|_1 \,.
	\end{align*}
	Note that as $\delta_s \rightarrow 0$ and $m_s \rightarrow \infty$, we get $\lambda \rightarrow \infty$ as expected. Next we compute the approximation errors. Using again Lemma \ref{lem:ms}, 
	\begin{align*}
		\bar \Delta(\alpha^*_\lambda) &= \Delta(\tilde \alpha^*) + \sum_{x \neq x^*} \tilde \alpha^*(x) \ip{\hat \theta_s - \thetaopt, x^* - x}\\
		&\leq c^*  +  \frac{\|\tilde \alpha^*\|_1 \beta_s^{1/2}}{(2m_s)^{1/2} + \beta_s^{1/2}} \leq c^* + 2 \|\tilde \alpha^*\|_1 \beta_s^{1/2} m_s^{-1/2} \,.
	\end{align*}
	To bound the approximation error of $I_s(\alpha^*_\lambda)$, note that $\beta_s = \|\hat \theta_s - \thetaopt\|_{V_s}^2 \leq \beta_{s,s^2}$ implies
	\begin{align*}
		I_s(\alpha^*_\lambda) &= \tfrac{1}{2}\sum_{z \in \xX} \alpha_\lambda^*(z) \sum_{x \neq x^*} q_s(x) \left(|\ip{\hat \nu_s(x) - \hat \theta_s, z}| + \beta_{s,s^2}^{1/2} \|z\|_{V_s^{-1}}\right)^2\\
		&\geq \tfrac{1}{2}\sum_{z \in \xX} \alpha_\lambda^*(z) \sum_{x \neq x^*} q_s(x) \ip{\hat \nu_s(x) - \thetaopt, z}^2\\
		&= \tfrac{1}{2}\sum_{x \neq x^*} q_s(x) \|\hat \nu_s(x) - \thetaopt\|_{V( \alpha_\lambda^*)}^2\\
		&\geq 1 - 2 C(\xX, \theta)\|\tilde \alpha^*\|_1\lambda^{-1}\,.
	\end{align*}
	The last step is by Lemma \ref{lem:truncated-optimal}. Finally, the proof is completed by using $\frac{c^* + A}{1 - B} \leq c^* + A + c^* B$, which yields
	\begin{align*}
		\Psi_s(\mu_s) \leq \frac{4 \delta_s \bar \Delta_s(\alpha_\lambda^*)}{I_s(\alpha_\lambda^*)} \leq 4 \delta_s \big(c^*  + 2 \|\tilde \alpha^*\|_1 \beta_s^{1/2}m_s^{-1/2} + 2 c^* C(\xX, \theta) \|\tilde \alpha^*\|_1 \lambda^{-1}\big) \,.
	\end{align*}
	Since $\lambda^{-1} = \oO\Big(c^{*-1} \big(\delta_s + 2\beta_s^{1/2} m_s^{-1/2}\big)\Big)$ for $\beta_s^{1/2}m_s^{-1/2} \rightarrow 0$ and $\delta_s \rightarrow 0$, we get
	\begin{align*}
		\Psi_s(\mu_s) \leq 4 \delta_s \big(c^*  + \oO(\beta_s^{1/2}m_s^{-1/2} + \delta_s )\big) \,.
	\end{align*}
\end{proof}

%% file: parts/appendix-infogain.tex
\subsection{Bounds on the Information Gain}

We start by proving a worst-case bound on the total information gain $\gamma_n = \sum_{s=1}^{s_n} I_s(x_s)$.
\begin{lemma}[Total information gain]\label{lem:info-bound-agnostic} 
	For any sequence $x_1, \dots, x_n$, the total information gain $\gamma_n = \sum_{s=1}^{s_n} I_{s}(x_s)$ is bounded as follows,
	\begin{align*}
		\gamma_n &\leq 2 \left(\beta_{s_n, n\log (n)} + \beta_{s_n,n\log(n)}^{1/2} + \beta_{s_n, s_n^2}\right) d \log(s_n) \leq \oO\big(d^2 \log(n)^2 \big)\,.
	\end{align*}
\end{lemma}
\begin{proof}
	Note that
	\begin{align*}
		\gamma_n = \sum_{s=1}^{s_n} I_s(x_s) &= \tfrac{1}{2}\sum_{s=1}^{s_n} \bigg(\sum_{x \neq \hat x_{s}} q_{s}(x) |\ip{\hat \nu_{s}(x) - \hat \theta_{s}, x_{s}}| + \beta_{s,s^2}^{1/2} \|x_s\|_{V_s^{-1}}\bigg)^2\\
		&\leq \sum_{s=1}^{s_n} \sum_{x \neq \hat x_{s}} q_{s}(x) \ip{\hat \nu_{s}(x) - \hat \theta_{s}, x_{s}}^2 + \beta_{s,s^2} \|x_s\|_{V_s^{-1}}^2\\
		&\stackrel{(i)}{\leq} \sum_{s=1}^{s_n} \sum_{x \neq \hat x_{s}} q_{s}(x) \big(\|\hat \nu_{s}(x) - \hat \theta_s\|_{V_s}^2 + \beta_{s,s^2}\big) \|x_{s}\|_{V_s^{-1}}^2\\
		&\stackrel{(ii)}{\leq} \sum_{s=1}^{s_n} \left( \min_{x \neq \hat x_{s}}  \|\hat \nu_s(x) - \hat \theta_s\|_{V_s}^2 + \frac{2\log(k)}{\eta_{s}} + \beta_{s,s^2}\right) \|x_{s}\|_{V_s^{-1}}^2 \\
		&\stackrel{(iii)}{\leq} \left(\beta_{s_n, n\log (n)} +  \beta_{s_n,n\log(n)}^{1/2} + \beta_{s_n, s_n^2}\right) \sum_{s=1}^{s_n} \|x_{s}\|_{V_s^{-1}}^2\\
		&\stackrel{(iv)}{\leq}  2 \left(\beta_{s_n, n\log (n)} + \beta_{s_n,n\log(n)}^{1/2} + \beta_{s_n, s_n^2}\right) d \log(s_n)
	\end{align*}
	Step $(i)$ uses Cauchy-Schwarz, $(ii)$ the soft-min bound for the $q$-weights (see Lemma \ref{lem:softmin}). For $(iii)$ we used that $m_s = \frac{1}{2}\min_{x \neq \hat x_s} \|\hat \nu_{s}(x) - \hat \theta_s\|_{V_s}^2 \leq \frac{1}{2} \beta_{s_n, n \log(n)}$ holds in all exploration rounds and the choice  $\eta_s = \lrate$ and lastly, (iv) bounds the elliptic potential (Lemma \ref{lem:elliptic-potential}). Considering that $\beta_{s,1/\delta} = 2 \log \frac{1}{\delta} + \oO(d \log s)$ completes the proof.
\end{proof}

\begin{lemma}[Constant information gain]\label{lem:const-information}
	Assume that $\hat x_s = x^*$ and $2 \delta_s \leq \hat \Delta_s(x)$ for all $x \neq \hat x_s$. If $z_s \neq x^*$ is contained in the support of the IDS distribution, $ \supp(\mu_s)$, then the information gain of $z_s$ is at least a constant,
	\begin{align*}
		I_s(z_s) \geq \frac{\Delta_{\min}^2}{8 (8d + 9)}\,.
	\end{align*}
\end{lemma}
\begin{proof}
	Note that by $z_s \in \supp(\mu_s)$ and  Lemma \ref{lem:ids-support},
	\begin{align*}
		I_s(z_s) &= \left(\hat \Delta_s(z_s) - \frac{\hat \Delta_s(\mu_s)}{2} \right) \frac{2 \hat \Delta_s(\mu_s)}{\Psi_s} \geq \left(\hat \Delta_s(z_s) - \delta_s \right) \frac{2 \delta_s}{\Psi_s} \geq \frac{\hat \Delta_s(z_s) \delta_s}{\Psi_s}\,.
	\end{align*}	
	We first used that $\delta_s \leq \hat \Delta_s(\mu) \leq 2 \delta_s$ (Lemma \ref{lem:almost-greedy}) and then the assumption that $2 \delta_s \leq \hat \Delta_s(z_s)$. Further, $2 \hat \Delta_s(z_s) \geq \Delta_{\min}$, and by Lemma \ref{lem:ratio-bound-instance},
	\begin{align*}
		\Psi_s(\mu_s) \leq \frac{4 \delta_s (8d + 9)}{\Delta_{\min}}\,.
	\end{align*}
	Combining the inequalities yields the result.
\end{proof}


\begin{lemma}\label{lem:asymptotic-information}
	Let $q_s^*(z) \propto \exp(-\eta_s \|\hat \nu_s(z) - \hat \theta_s\|_{V_s}^2)$ be mixing weights defined on $\xX \setminus x^*$ (also when $\hat x_s \neq x^*$), where  $\hat \nu_s(z) = \argmin_{\nu \in \hh{x^*}{z}} \|\nu - \hat \theta_s\|_{V_s}^2$ for all $z \neq x^*$. Define $l_s(q_s) = \sum_{z \neq x^*} q_s^*(z) \ip{\hat \nu_s(z) - \theta_s}^2$ and let $J_s = \chf{24^2 \eta_s \beta_s\|x_s\|_{V_s^{-1}}^2 \leq 1; \beta_s \|x_s\|_{V_s^{-1}}^2 \leq 1}$. Then
	\begin{align*}
		\EE\Big[\tsum_{s=1}^{s_n} J_s l_s(q_s^*) - \min_{x \neq x^*} \|\hat \nu_{s_n}(x) - \hat \theta_{s_n}\|_{V_{s_n}}^2\Big] \leq  \oO\big( \log(n)^{1/2}\EE[\log(s_n)^2]\big) \,.
	\end{align*}
\end{lemma}
\begin{proof}
	The statement is a regret bound for the exponential weights learner that defines the $q_s^*$-weights,  excluding steps where $J_s=0$. The difference to standard online learning bounds is that the cumulative loss $L_s(x) = \frac{1}{2}\|\hat \nu_s(x) - \hat \theta_s\|_{V_s}^2$, which defines the mixing weights and the baseline, does not exactly equal the sum of instantaneous loss $\sum_{s=1}^{s_n} l_s(x)$. For the analysis we make use of well-known connections between the exponential weights algorithm and the mirror descent framework, in particular the \emph{follow the regularized leader} (FTRL) algorithm \citep{shalev2007online}. To this end, let $\psi(q) = \sum_{x \neq x^*} q(x) \log(q(x))$ be the entropy function defined for $q \in \sP(\xX \setminus \hat x_s)$. For learning rate $\eta > 0$, we define
	\begin{align*}
		\psi_\eta(q) = \frac{1}{\eta}\left(\psi(q) - \min_{q' \in \sP(\xX \setminus x^*)} \psi(q')\right) \,.
	\end{align*}
	We denote $\psi_s = \psi_{\eta_s}$.	The choice of mixing weights $q_s^*$ can be equivalently written as
	\begin{align*}
		q_s^* = \argmin_{q \in \sP(\xX \setminus x^*)} L_s(q) + \psi_{s}(q)\,. 
	\end{align*}
	Denote $\Lambda_n = \sum_{s=1}^{s_n} J_s l_s(q_s^*) - \min_{x \neq x^*} \|\hat \nu_{s_n}(x) - \hat \theta_{s_n}\|_{V_{s_n}}^2$. The following inequality is easily verified by telescoping \citep[c.f. Lemma 7.1][]{orabona2019modern},
		\begin{align*}
		\Lambda_n \leq - \frac{1}{\eta_{s_n}} \min_{q'} \psi(q') + \sum_{s=1}^{s_n} \big([L_{s} + J_s l_s + \psi_s](q_s) - [L_{s+1} + \psi_{s+1}](q_{s+1}) \big)\,.
	\end{align*}
	For the first term, we immediately get $-\frac{1}{\eta_s} \min_{q'} \psi(q') \leq \frac{\log(k)}{\eta_{s_n}}$. 
	The second term is often referred to as stability term.
	We first address steps $s$ where $J_s=1$. Define $\tilde q_{s+1} = \argmin_{q \in \sP(\xX \setminus x^*)}[L_{s+1} + \psi_{s}](q) \propto \exp(-\eta_s L_{s+1})$. Using that the learning rate is decreasing, we get
	\begin{align}
	&\+[L_{s} + l_s + \psi_s](q_s) - [L_{s+1} + \psi_{s+1}](q_{s+1}) \nonumber\\
		&\leq [L_{s + 1} + \psi_s](q_s) - [L_{s+1} + \psi_{s}](\tilde q_{s+1}) + [L_{s} + l_s - L_{s+1}](q_s) \,. \label{eq:stability}
	\end{align}
	Note that $L_{s+1}$ exhibits an intricate dependence on the outcome $y_s$, whereas all other quantities appearing in the last display are $\fF_s$-predictable. Using that $\tilde q_{s+1}$ is a minimizer of $L_{s+1} + \psi_s$ and the definition of the Bregman divergence $D_{\psi}(p\|q) = \psi(p) - \psi(q) -\ip{\nabla \psi(q), p-q}$, we find
	\begin{align*}
		[L_{s + 1} + \psi_s](q_s^*) - [L_{s+1} + \psi_{s}](\tilde q_{s+1}) = \frac{1}{\eta_s} D_{\psi_s}(q_s^*, \tilde q_{s+1}) = \frac{1}{\eta_s}\sum_{x \neq x^*} q_s^*(x) \log \frac{q_s^*(x)}{\tilde q_{s+1}(x)}
	\end{align*}
	Using that $\log(x) \leq x - 1$ for all $x > 0$, we find
	\begin{align*}
		\sum_{x \neq x^*} q_s^* \log \frac{q_s^*}{\tilde q_{s+1}} &= \eta_s [L_{s+1} - L_{s}](q_s^*) + \log \bigg(\sum_{x \neq x^*} q_s^* \exp\big(-\eta_s (L_{s+1} - L_{s})\big)\bigg)\\
		&\leq -1 + \eta_s [L_{s+1} - L_{s}](q_s^*) + \sum_{x \neq x^*} q_s^* \exp\big(-\eta_s (L_{s+1} - L_{s}))\\
		&=  \sum_{x \neq x^*} q_s^*(x)\sum_{i=2}^\infty \frac{(-\eta_s (L_{s+1} - L_{s}))^i}{i!}
	\end{align*}
	A technical calculation which directly bounds the moments of the subgaussian noise under the conditional expectation $\EE[\cdot |\fF_s]$ with the condition $J_s=1$, is summarized in Lemma \ref{lem:exponential-sum}. This yields
	\begin{align*}
		&\+ \sum_{s=1}^{s_n} J_s \EE\B[ [L_{s + 1} + \psi_s](q_s^*) - [L_{s+1} + \psi_{s}](\tilde q_{s+1})\B|\fF_s\B]\\
		&\leq \sum_{s=1}^{s_n}\frac{J_s}{\eta_s} \sum_{x \neq x^*} q_s^*(x)\EE\B[\sum_{i=2}^\infty \frac{(-\eta_s (L_{s+1}(x) - L_{s}(x)))^i}{i!}\B|\fF_s\B]\\
		&\leq \sum_{s=1}^{s_n} \sum_{x \neq x^*} q_s^*(x) \oO\Big(\eta_s \big(\beta_s \|x_s\|_{V_s^{-1}}^2 + \|\hat \nu_s(x) - \hat \theta_s \|_{V_s}^2 \|x_s\|_{V_s^{-1}}^2 \big)\Big)\\
		&\leq \oO\Big(\log(n)^{1/2}\log(s_n)^2\Big)
	\end{align*}
	The last step makes use of Lemma \ref{lem:softmin}, $\eta_s m_s \leq \beta_{s_n,n \log n}^{1/2} \leq \oO(\log(n)^{1/2} + \log(s_n)^{1/2})$ and Lemma \ref{lem:softmin}.
	Going back to \eqref{eq:stability}, still for the case where $J_s=1$, it remains to bound the shift term $S_s = L_s + l_s - L_{s+1}$. We have
	\begin{align*}
		\EE[S_s(q_s^*)|\fF_s] &\stackrel{(i)}{\leq} 2 \|x_s\|_{V_s^{-1}}^2 \big( \tsum_{x \neq x^*} q_s\|\hat \nu_s(x)-\hat \theta_s\|_{V_s}\beta_s^{1/2} + \beta_s + 1\big)\\
		&\stackrel{(ii)}{\leq} 2 \|x_s\|_{V_s^{-1}}^2 \Big( \sqrt{\tsum_{x \neq x^*} q_s\|\hat \nu_s(x)-\hat \theta_s\|_{V_s}^2}\beta_s^{1/2} + \beta_s + 1\Big)\\
		&\stackrel{(iii)}{\leq} 2 \|x_s\|_{V_s^{-1}}^2 \Big( \big((m_s + \log(k)/\eta_{s}) \beta_s\big)^{1/2} + \beta_s + 1\Big)
	\end{align*}
	Here, (i) follows from the Lemma \ref{lem:shift-term}, Cauchy-Schwarz and taking the expectation; (ii) is Jensen's inequality and (iii) is the softmin inequality (Lemma \ref{lem:softmin}).
	Hence, using that $m_s \leq \beta_{s_n, n\log(n)} \leq \oO(\log(n) + \log(s_n))$ and the elliptic potential lemma (Lemma \ref{lem:elliptic-potential}), we find
	\begin{align*}
		\sum_{s=1}^{s_n} \EE[S_s(x)|\fF_s] &\leq \oO\big(\log(s_n)^2 \log(n)^{1/2}\big)\,.
	\end{align*}
	Lastly, we address \eqref{eq:stability} for the case $J_s = 0$, which then reads
	\begin{align}
			&\+[L_{s} + \psi_s](q_s) - [L_{s+1} + \psi_{s+1}](q_{s+1}) \leq [L_{s} - L_{s+1}](q_{s+1})\,.
	\end{align} 
	We can reuse Lemma \ref{lem:difference-term} to find
	\begin{align*}
		 \EE_s[[L_{s} - L_{s+1}](x)] &\leq \oO\big(\beta_s \|x_s\|_{V_s^{-1}}^2  +  |\ip{\hat \nu_{s} - \hat \theta_{s},  x_s}|  + \ip{\hat \nu_s - \hat \theta_s,x_s}^2\big)\\
		  &\leq \oO\big(\beta_s \|x_s\|_{V_s^{-1}}^2  + 1)\big)
	\end{align*}
	Using that when $J_s=0$ we have $1 \leq \beta_s \|x_s\|_{V_s}^2$, or $1 \leq 24^2 \eta_s \beta_s \|x_s\|_{V_s^{-1}}^2$, we can sum up these terms to 
	\begin{align*}
		\sum_{s=1}^{s_n} \EE_s[L_s - L_{s+1}](x) \leq \sum_{s=1}^{s_n} \oO(\beta_s \|x_s\|_{V_s^{-1}}^2) \leq \oO(\log(s_n)^2)
	\end{align*}
	The claim follows.
\end{proof}

\begin{lemma}\label{lem:shift-term}
	Let $L_s(x) = \|\hat \nu_s(x) - \hat \theta_s\|_{V_s}^2$ defined for $x \neq x^*$ and assume that $\ip{\nu - \thetaopt, x} \leq 1$ for all $\nu \in \mM$ and $x \in \xX$. Then
	\begin{align*}
		[L_s + l_s - L_{s+1}](x) \leq 2 \ip{\hat \nu_s(x) - \hat \theta_s, x_s}\frac{\epsilon_s + \ip{x_s, \thetaopt - \hat \theta_s}}{1 + \|x_s\|_{V_{s}^{-1}}^2} + 2 \|x_s\|_{V_s^{-1}}^2(1 + \beta_s)
	\end{align*}
\end{lemma}

\begin{proof}
	For the proof  we adopt the notation $\omega_s(x) = \hnu_s(x) - \htheta_s$.	
	\begin{align*}
		L_s + l_s - L_{s+1} &=\|\omega_s\|_{V_{s+1}}^2  - \|\omega_{s+1}\|_{V_{s+1}}^2 \\
		&=\|\omega_s\|_{V_{s+1}}^2  - \|\omega_s  + \omega_{s+1} -\omega_s\|_{V_{s+1}}^2 \\
		&= 2 \ip{\omega_s - \omega_{s+1},V_{s+1}\omega_s} - \| \omega_{s+1} -\omega_s\|_{V_{s+1}}^2\\
		&= \underbrace{2 \ip{\omega_{s} -\omega_{s+1},V_{s}\omega_s}}_{(A)} + \underbrace{2 \ip{\omega_{s} -\omega_{s+1},x_s}\ip{x_s,\omega_s} - \| \omega_{s+1} -\omega_{s}\|_{V_{s+1}}^2}_{(B)} 
	\end{align*}
	To avoid clutter, the dependence on $x$ is implicit below. Note that because $\hat \nu_s$ is a projection of $\hat \theta_{s}$ $V_s$-norm onto the convex set $\hh{x^*}{x}$, we have $\ip{\hat \nu_s - \hat \nu_{s+1} ,V_{s}(\hat \nu_s - \hat \theta_s)} \leq 0$. Therefore
	\begin{align*}
		(A) \leq 2 \ip{\hat \theta_{s+1} - \hat \theta_s ,V_{s}(\hat \nu_s - \hat\theta_s)} = 2 \ip{\hat \nu_s - \hat \theta_s, x_s}\frac{\epsilon_s + \ip{x_s, \thetaopt - \hat \theta_s}}{1 + \|x_s\|_{V_{s}^{-1}}^2}
	\end{align*}
	The equality follows from Lemma \ref{lem:ls-sherman}. Next, we derive an upper bound to the term (B).
	\begin{align*}
		(B) &\leq 2 \ip{\omega_{s} -\omega_{s+1},x_s}\ip{x_s,\omega_s} - \| \omega_{s+1} -\omega_{s}\|_{V_{s+1}}^2\\
		&\leq 2 \|\omega_{s} - \omega_{s+1}\|_{V_s}\|x_s\|_{V_s^{-1}} \ip{x_s,\omega_s} - \| \omega_{s+1} -\omega_{s}\|_{V_{s+1}}^2\\
		&\leq 2 \|\omega_{s} - \omega_{s+1}\|_{V_{s+1}}\|x_s\|_{V_s^{-1}} \ip{x_s,\omega_s} - \| \omega_{s+1} -\omega_{s}\|_{V_{s+1}}^2\\
		&\leq \|x_s\|_{V_s^{-1}}^2 \ip{x_s,\omega_s}^2 \leq 2 \|x_s\|_{V_s^{-1}}^2 (1+\beta_s)
	\end{align*}
	We used Cauchy-Schwarz and $\|\cdot\|_{V_s}^2 \leq \|\cdot\|_{V_{s+1}}^2$ in the first and second inequality. Then we use $2ab - b^2 \leq a^2$, and in the last step boundedness, $|\ip{\omega_s(x),x_s}|  \leq \ip{\hat \nu_s(x) - \thetaopt,x_s}| + \beta_s^{1/2}\|x_s\|_{V_s^{-1}} \leq 1 + \beta_s^{1/2}$. The claim follows from combining the bounds.
\end{proof}

\begin{lemma}\label{lem:difference-term}
	Let $L_s(x) = \|\hat \nu_s(x) - \hat \theta_s\|_{V_s}^2$ defined for $x \neq x^*$ and assume that $\ip{\nu - \thetaopt, x} \leq 1$ for all $\nu \in \mM$ and $x \in \xX$. Then
	\begin{align*}
		|[L_{s} - L_{s+1}](x)|& \leq 4 |\epsilon|^2\|x_s\|_{V_s^{-1}}^2  
		+  2|\ip{\hat \nu_{s} - \hat \theta_{s},  x_s}||\epsilon_s|  + 8 \beta_s \|x_s\|_{V_s^{-1}}^2 +  \ip{\hat \nu_s - \hat \theta_s,x_s}^2
	\end{align*}
\end{lemma}

\begin{proof}
	For one direction, we can reuse Lemma \ref{lem:shift-term},
	\begin{align*}
		[L_{s} - L_{s+1}](x) &\leq [L_{s} + l_s - L_{s+1}](x)\\
		&\leq 2 |\epsilon_s| |\ip{\hat \nu_s(x) - \hat \theta_s, x_s}| + 2 \|x_s\|_{V_s^{-1}}\beta_s^{1/2} + 2 \|x_s\|_{V_s^{-1}}^2(1 + \beta_s)\,.
	\end{align*}
	For the other direction, we have
	\begin{align*}
		[L_{s+1} - L_s](x) &= \|\hat \nu_{s+1} - \hat \theta_{s+1}\|_{V_{s+1}}^2 - \|\hat \nu_{s} - \hat \theta_{s}\|_{V_s}^2\\
		&\leq\|\hat \nu_{s} - \hat \theta_{s+1}\|_{V_{s+1}}^2 - \|\hat \nu_{s} - \hat \theta_{s}\|_{V_s}^2\\
		&=\|\hat \nu_{s} - \hat \theta_{s} + V_s^{-1} x_s u_s\|_{V_{s+1}}^2 - \|\hat \nu_{s} - \hat \theta_{s}\|_{V_s}^2\,,
	\end{align*}
	where for the last step we denote $u_s = \frac{\epsilon_{s} + \ip{x_s, \thetaopt - \hat \theta_s}}{1 + \|x_s\|_{V_s^{-1}}^2}$ and use Lemma \ref{lem:ls-sherman}. Further unwrapping the square gives
	\begin{align*}
		&\+\|\hat \nu_{s} - \hat \theta_{s} - V_s^{-1} x_s u_s\|_{V_{s+1}}^2 - \|\hat \nu_{s} - \hat \theta_{s}\|_{V_s}^2\\
		&=\|\hat \nu_{s} - \hat \theta_{s} - V_s^{-1} x_s u_s\|_{V_s}^2 + \ip{ \hat \nu_s - \hat \theta_s - V_s^{-1} x_s u_s, x_s}^2 - \|\hat \nu_s - \hat \theta_s\|_{V_s}^2\\
		&= - 2\ip{\hat \nu_{s} - \hat \theta_{s},  x_s} u_s + u_s^2 \|x_s\|_{V_s^{-1}}^2 + \ip{\hat \nu_s - \hat \theta_s,x_s}^2 - 2\ip{\hat \nu_s - \hat \theta_s,x_s} \|x_s\|_{V_s^{-1}}^2 +  \|x_s\|_{V_s^{-1}}^4 u_s^2\\
		&\leq - 2\ip{\hat \nu_{s} - \hat \theta_{s},  x_s} u_s(1+ \|x_s\|_{V_s^{-1}}^2) + 2 u_s^2 \|x_s\|_{V_s^{-1}}^2 + \ip{\hat \nu_s - \hat \theta_s,x_s}^2 \\
		&\leq 2|\ip{\hat \nu_{s} - \hat \theta_{s},  x_s}|(|\epsilon_s + \beta_s^{1/2} \|x_s\|_{V_s^{-1}}) + 4 (|\epsilon|^2 + \beta_s \|x_s\|_{V_s^{-1}}^2) \|x_s\|_{V_s^{-1}}^2 + \ip{\hat \nu_s - \hat \theta_s,x_s}^2\\
		&\leq 2|\ip{\hat \nu_{s} - \hat \theta_{s},  x_s}||\epsilon_s| + 2 \beta_s^{1/2} \|x_s\|_{V_s^{-1}} + 4 |\epsilon|^2\|x_s\|_{V_s^{-1}}^2  
		+ 6 \beta_s \|x_s\|_{V_s^{-1}}^2 + \ip{\hat \nu_s - \hat \theta_s,x_s}^2
	\end{align*}
	Combining both directions yields the claim.
\end{proof}

\begin{lemma}\label{lem:exponential-sum}
	Let $s$ such that $24^2 \eta_s \beta_s \|x_s\|_{V_s^{-1}}^2 \leq 1$ and $\beta_s \|x_s\|_{V_s^{-1}}^2 \leq 1$. Then
	\begin{align*}
	\EE\B[\sum_{i=2}^\infty \frac{|\eta_s (L_{s+1}(x) - L_{s}(x))|^i}{i!}\B|\fF_s\B] \leq  \oO\Big(\eta_s^2 \big(\beta_s \|x_s\|_{V_s^{-1}}^2 + \|\hat \nu_s(x) - \hat \theta_s \|_{V_s}^2 \|x_s\|_{V_s^{-1}}^2 \big)\Big) \,.
	\end{align*}
\end{lemma}
\begin{proof}
\begin{align*}
  &\+| (L_{s+1}(x) - L_{s}(x))|^i|\\
   &\leq \big(4 |\epsilon|^2\|x_s\|_{V_s^{-1}}^2  +  2|\ip{\hat \nu_{s} - \hat \theta_{s},  x_s}||\epsilon_s|  + 8 \beta_s \|x_s\|_{V_s^{-1}}^2 +  \ip{\hat \nu_s - \hat \theta_s,x_s}^2\big)^i\\
   &\leq \big(12 |\epsilon|^2\|x_s\|_{V_s^{-1}}^2\big)^i  +  \big(6|\ip{\hat \nu_{s} - \hat \theta_{s},  x_s}||\epsilon_s|\big)^i  + \big(24 \beta_s \|x_s\|_{V_s^{-1}}^2 +  3\ip{\hat \nu_s - \hat \theta_s,x_s}^2\big)^i
\end{align*}
For the last step we used that for $a,b,c, \geq 0$,  $(a + b + c)^i \leq (3a)^i + (3b)^i + (3c)^i$. Further note that for the $\sigma$-subgaussian noise $\epsilon_s$, it holds that for all $i \in \NN$, $\EE[|\epsilon|^i] \leq (2\sigma^2)^{i/2} i \Gamma(i/2) \leq (2\sigma^2)^{i} i! $ and $\EE[|\epsilon|^2i] \leq (2\sigma^2)^i 2i!$ \citep[c.f. Lemma 1.4,][]{rigollet2015lecture}.
Hence we get
\begin{align*}
\EE_s\left[\frac{|\eta_s (L_{s+1}(x) - L_{s}(x))|^i}{i!}\right] &\leq \EE_s\left[\frac{(12 \eta_s |\epsilon|^2\|x_s\|_{V_s^{-1}}^2)^i}{i!}\right]  +  \EE_s\left[\frac{(6 \eta_s |\ip{\hat \nu_{s} - \hat \theta_{s},  x_s}||\epsilon_s|)^i}{i!}\right]\\
	&\qquad   + \frac{(24 \eta_s \beta_s \|x_s\|_{V_s^{-1}}^2 +  3 \eta_s \ip{\hat \nu_s - \hat \theta_s,x_s}^2)^i}{i!}
\end{align*}
We address each term individually, also using that $24^2 \eta_s \beta_s \|x_s\|_{V_s^{-1}}^2 \leq 1$.
\begin{align*}
	\EE_s\left[\frac{(12 \eta_s |\epsilon|^2\|x_s\|_{V_s^{-1}}^2)^i}{i!}\right] &\leq (24 \eta_s \sigma^2 \|x_s\|_{V_s^{-1}}^2)^i\\
	&\leq (24 \eta_s \sigma^2 \|x_s\|_{V_s^{-1}}^2)^2 \cdot 2^{-i+2}\\
 \EE_s\left[\frac{(6 \eta_s |\ip{\hat \nu_{s} - \hat \theta_{s},  x_s}||\epsilon_s|)^i}{i!}\right] & \leq  (12 \eta_s |\ip{\hat \nu_{s} - \hat \theta_{s},  x_s}| \sigma^2)^i\\
 & \leq (12 \eta_s |\ip{\hat \nu_{s} - \hat \theta_{s},  x_s}| \sigma^2)^2 \cdot 2^{-i+2}\\
 \frac{(24 \eta_s \beta_s \|x_s\|_{V_s^{-1}}^2 +  3 \eta_s \ip{\hat \nu_s - \hat \theta_s,x_s}^2)^i}{i!} & \leq (24 \eta_s \beta_s \|x_s\|_{V_s^{-1}}^2 +  3 \eta_s \ip{\hat \nu_s - \hat \theta_s,x_s}^2)^{i-2} \frac{2^{i-2}}{i!}
\end{align*}
Summing over $i=2,\dots, \infty$ gives
\begin{align*}
	&\+\sum_{i=2}^\infty\EE_s\left[\frac{|\eta_s (L_{s+1}(x) - L_{s}(x))|^i}{i!}\right] \\
	&\leq \oO\Big( \big(\eta_s \|x_s\|_{V_s^{-1}}^2\big)^2 + \big(\eta_s |\ip{\hat \nu_{s} - \hat \theta_{s},  x_s}| \big)^2 + \big(\eta_s \beta_s \|x_s\|_{V_s^{-1}}^2 +  \eta_s \ip{\hat \nu_s - \hat \theta_s,x_s}^2\big)^2 \Big)\\
	&\leq \oO\Big(\eta_s^2 \big(\beta_s \|x_s\|_{V_s^{-1}}^2 + \|\hat \nu_s(x) - \hat \theta_s \|_{V_s}^2 \|x_s\|_{V_s^{-1}}^2 \big)\Big)
\end{align*}
For the last step we summarize the terms using also that for $J_s=1$, we have $\beta_s \|x_s\|_{V_s^{-1}}^2 \leq 1$.
\end{proof}

%% file: parts/appendix-asymptotic.tex
\subsection{Asymptotic Regret: Proof of Theorem \ref{thm:regret-asymptotic}}\label{app:asymptotic}

\newcommand{\Si}{S_1}
\newcommand{\Sii}{S_2}
\newcommand{\Siii}{S_3}

\begin{proofof}{Theorem~\ref{thm:regret-asymptotic}}
	As before, we let $\beta_s = \|\hat \theta_s - \theta^*\|_{V_s}^2$ and $B_s = \chf{\beta_s \leq \beta_{s, s^2}}$. With Lemma~\ref{lem:exploration-regret} we get
	\begin{align*}
		\EE[R_n] \leq \EE\Big[\tsum_{s=1}^{s_n} \Delta(x_s)B_s\Big] + \oO\big(\log \log(n)\big)
	\end{align*}
	Recall that $m_s = \frac{1}{2}\min_{x\neq \hat x_s} \|\hat \nu_s(x) -\hat \theta_s\|_{V_s}^2$.
	Let $\lambda$ be a trade-off parameter, which in hindsight is chosen as $\lambda = \log(n)^{-2/3} \leq \frac{1}{4}$ for $n$ large enough.
	We decompose the exploration rounds into three disjoint sets, which capture different regimes as $\beta_{s,s^2}/m_s \rightarrow 0$ and $\delta_s \rightarrow 0$:
	\begin{align*}
		\Si &= \left\{s \in [s_n] : \tfrac{\beta_{s,s^2}}{m_s} > \lambda,  \beta_s \leq \beta_{s,s^2}\right\}\\
		\Sii &= \left\{s \in [s_n] : \tfrac{\beta_{s,s^2}}{m_s} \leq \lambda, \tfrac{\delta_s^2}{16} > \tfrac{\beta_{s,s^2}}{m_s}, \beta_s \leq \beta_{s,s^2}\right\}\\
		\Siii &= \left\{s \in [s_n] :\tfrac{\delta_s^2}{16} \leq \tfrac{\beta_{s,s^2}}{m_s} \leq \lambda,  \beta_s \leq \beta_{s,s^2}\right\}
	\end{align*}
	In particular, we can write
	\begin{align*}
		\EE\Big[\tsum_{s=1}^{s_n} \Delta(x_s)B_s\Big] = \EE\Big[\tsum_{s \in \Si} \Delta(x_s) \Big] + \EE\Big[\tsum_{s \in \Sii} \Delta(x_s) \Big] + \EE\Big[\tsum_{s \in \Siii} \Delta(x_s) \Big] \,.
	\end{align*}
	We address the three terms in order.
	\paragraph{Sum over $\bm \Si$:}
	Cauchy-Schwarz and a few applications of the tower rule as before show that 
	\begin{align*}
		\EE\big[\tsum_{s \in \Si} \hat \Delta(x_s)\big]^2 \leq \EE\left[\tsum_{s \in \Si} \Psi_s\right] \EE\left[\tsum_{s \in \Si} I_s(x_s)\right] \,.
	\end{align*}
	To bound the information-ratio, the definition of $\Si$ implies the conditions of Lemma~\ref{lem:ratio-bound-instance}, which combined with 
	$\delta_s \leq \hat \Delta_s(x_s)$ yields 
	\begin{align*}
		\sum_{s \in \Si} \EE\big[\Psi_s\big] \leq  \oO\left(\tfrac{d}{\Delta_{\min}}\right) \sum_{s \in \Si} \EE\big[\hat \Delta_s(x_s)\big]\,.
	\end{align*}
	The total information gain on $\Si$ is bounded using the same steps as in the proof of Lemma \ref{lem:info-bound-agnostic},
	\begin{align*}
		\sum_{s \in \Si} I_s(x_s) &\leq  \sum_{s \in \Si} \big(m_s + \tfrac{\log(k)}{\eta_s} + \beta_{s,s^2} \big)\|x_s\|_{V_s^{-1}}^2\\
		&\stackrel{(i)}{\leq}  \sum_{s \in \Si} \big(\beta_{s,s^2}(\lambda^{-1} + 1) + \tfrac{\log(k)}{\eta_s} \big)\|x_s\|_{V_s^{-1}}^2\\
		&\stackrel{(ii)}{\leq} \oO\big(\lambda^{-1}d^2\log(s_n)^2 +  d^{3/2}\log(n)^{1/2}\log(s_n)\big) \,,
	\end{align*}
	where $(i)$ follows because $m_s < \beta_{s,s^2} \lambda^{-1}$ for $s \in \Si$ and $(ii)$ from the elliptic potential (Lemma~\ref{lem:elliptic-potential}) and using that $\log(k)\eta_s^{-1} \leq \beta_{s_n, n\log(n)}^{1/2}$. Combining and rearranging 
	the last three displays and using $\Delta(x_s)B_s \leq 2 \hat \Delta_s(x_s)B_s$ with $B_s = 1$ for $s \in \Si$ yields
	\begin{align*}
		\EE\B[\sum_{s \in \Si} \Delta(x_s)\B] \leq \oO\Big(\lambda^{-1} \Delta_{\min}^{-1}d^3 \EE[\log(s_n)^2] + \Delta_{\min}^{-1} d^{5/2} \log(n)^{1/2} \EE[\log(s_n)]\Big)\,.
	\end{align*}
	\paragraph{Sum over $\bm \Sii$:} First note that $\beta_s \leq \beta_{s,s^2} < m_s$ implies $\hat x_s = \hat x^{\UCB} =  x^*$. For any $x \in \xX$,
	\begin{align}
		\beta_{s,s^2}^{-1/2}\delta_s - \|x\|_{V_s^{-1}} 
		&\stackrel{(i)}{=} \|x^*\|_{V_s^{-1}} - \|x\|_{V_s^{-1}} 
		\stackrel{(ii)}{\leq} \|x^* - x \|_{V_s^{-1}} \nonumber\\
		&\stackrel{(iii)}{\leq} \frac{1}{(2m_s)^{1/2}  - \beta_{s}^{1/2}} 
		\stackrel{(iv)}{\leq} \frac{2} {m_s^{1/2}} 
		\stackrel{(v)}{<} \frac{\delta_s}{2\beta_{s,s^2}^{1/2}}\,, \label{eq:ms-lemma-applied}
	\end{align}
	where $(i)$ follows because $\hat x_s = x_s^{\UCB} = x^*$, implying that $\delta_s = \beta_{s,s^2}^{1/2} \|x^*\|_{V_s^{-1}}$.
	$(ii)$ follows from the triangle inequality, $(iii)$ from Lemma~\ref{lem:ms} and $(iv)$ because $\beta_s \leq m_s / 4$. Finally, $(v)$ holds since 
	$\delta_s^2 / 16 > \beta_{s,s^2} / m_s$. 
	With $x = x_s$ and rearranging yields $\delta_s \leq 2 \beta_{s,s^2}^{1/2}\|x_s\|_{V_s^{-1}}$ and hence
	\begin{align*}
		\sum_{s \in \Sii} \EE[\hat \Delta_s(x_s)] 
		&= \sum_{s \in \Sii} \EE[\hat \Delta_s(\mu_s)] 
		\stackrel{(i)}{\leq} 2 \sum_{s \in \Sii} \EE[\delta_s] 
		\leq 4\sum_{s \in \Sii} \EE[\beta_{s,s^2}^{1/2}\|x_s\|_{V_s^{-1}}]\,,
	\end{align*}
	where (i) uses $\hat \Delta_s(\mu_s) \leq 2 \delta_s$ (Lemma \ref{lem:almost-greedy}). From here, we can apply Cauchy-Schwarz in a similiar manner as before, to get
	\begin{align*}
		\EE\B[\sum_{s \in \Sii} \beta_{s,s^2}^{1/2} \|x_s\|_{V_s^{-1}}\B]^2 &\leq \EE\B[\sum_{s \in \Sii} \beta_{s,s^2}^{1/2}m_s^{-1/2}\B] \EE\B[\sum_{s \in \Sii} m_s^{1/2} \beta_{s,s^2}^{1/2} \|x_s\|_{V_s^{-1}}^2\B]\\
		& \leq \EE\B[\sum_{s \in \Sii} \hat \Delta_s(x_s)\B] \oO\big(d^{2}\log(n)^{1/2}\EE[\log(s_n)^2]\big)\,.
	\end{align*}
	For the last inequality, we used that $4\beta_{s,s^2}^{1/2} m_s^{-1/2} \leq \delta_s \leq \hat \Delta_s(x_s)$, the elliptic potential (Lemma \ref{lem:elliptic-potential}) and $m_s \leq \beta_{s_n, n \log (n)} \leq \oO(\log(n) + d \log(s_n))$.
	Hence, combining the last two displays and  $\Delta_s(x_s) \leq 2 \hat \Delta_s(x_s)$, we get
	\begin{align*}
		\EE\B[\sum_{s \in \Sii} \Delta(x_s)\B] \leq\oO\big(d^{2}\log(n)^{1/2}\EE[\log(s_n)^2]\big) \,.
	\end{align*}
	\paragraph{Sum over $\bm \Siii$:} Denote $\bar \Delta_s(x) = \ip{\hat \theta_s, \hat x_s - x}$. Note that $\hat x_s = x^*$ continues to hold, and hence
	\begin{align}
		\EE\Big[\tsum_{s \in \Siii} \Delta(x_s)\Big] \leq \EE\Big[\tsum_{s \in \Siii} \bar \Delta_s(x_s)\Big]  + \EE \Big[\tsum_{s \in \Siii}\beta_s^{1/2} \|x^* - x_s\|_{V_s^{-1}}\Big]\,. \label{eq:main-split}
	\end{align}
	For the second sum, note that by Lemma \ref{lem:const-information} the information gain of $x_s \neq x^*$ is lower bounded by a constant, $I_s(x_s) \geq \Omega\left(\frac{\Delta_{\min}^2}{d}\right)$. As in \eqref{eq:ms-lemma-applied}, Lemma \ref{lem:ms} implies 
	\begin{align*}
		\beta_s^{1/2} \|x^* - x_s\|_{V_s^{-1}} \leq 2\beta_{s}^{1/2} m_s^{-1/2}\chf{x_s \neq x^*}  \leq  \oO\big(\lambda^{1/2} d \Delta_{\min}^{-2} I_s(x_s)\big)\,.
	\end{align*}
	Summing the last display inside the expectation and using Lemma \ref{lem:info-bound-agnostic} yields
	\begin{align*}
		\EE\Big[\tsum_{s \in \Siii} \beta_s^{1/2} \|x^* - x_s\|_{V_s^{-1}}\Big] \leq \oO\big(\lambda^{1/2} d\log(n) \EE[\log(s_n)]\big)\,.
	\end{align*}
	For the first sum in \eqref{eq:main-split}, we use $4ab \leq (a + b)^2$ and Cauchy-Schwarz combined with a few applications of the towering rule, to get
	\begin{align}
		\EE\bigg[\sum_{s \in \Siii} \bar \Delta_s(\mu_s)\bigg] 
		&\leq \frac{1}{4} \EE\bigg[\sum_{s \in \Siii}\delta_s\bigg]^{-1} \EE\bigg[\sum_{s \in \Siii} \hat \Delta_{s }(\mu_s)\bigg]^2 \nonumber \\
		&\leq \frac{1}{4} \EE\bigg[\sum_{s \in \Siii}\delta_s\bigg]^{-1} \EE\bigg[\sum_{s \in \Siii} \Psi_{s}(\mu_s)\bigg]\EE\bigg[\sum_{s \in \Siii} I_s(x_s)\bigg] \label{eq:S3-CS}
	\end{align}
	Lemma \ref{lem:ratio-bound-asymptotic} bounds the information ratio, $\Psi_s(\mu_s) \leq 4 \delta_s (c^* + \oO(\delta_s + \beta_s^{1/2}m_s^{-1/2}) \leq 4 \delta_s (c^* + \oO(\lambda))$, making use of $\delta_s/4 \leq \beta_{s,s^2}^{1/2}m_s^{-{1/2}} \leq \lambda^{1/2}$. In particular,
	\begin{align*}
		\frac{1}{4} \EE\bigg[\sum_{s \in S_3}\delta_s\bigg]^{-1} \EE\bigg[\sum_{s \in S_3} \Psi_{s}(\mu_s)\bigg] \leq c^* + \oO(\lambda^{1/2})
	\end{align*}
	To bound the information gain on $S_3$, denote $l_s(q_s) = \sum_{x \neq x^*} q_s(x) \ip{\hat \nu_s(x) - \hat \theta_s, x_s}^2$. Note that since $\hat x_s = x^*$ on $\Siii$, $l_s(q_s) = I_s(x_s)$. Further, let $J_s = \chf{24^2 \eta_s \beta_{s} \|x_s\|_{V_s^{-1}}^2 \leq 1; \beta_s \|x_s\|_{V_s^{-1}}^2 \leq 1}$. It is easy to verify that for small enough $\lambda$, $J_s=1$ for all $s \in \Siii$. Hence, by Lemma \ref{lem:asymptotic-information} and $m_s \leq \log(n) + \log\log(n) + \oO(d\log(s_n))$,
	\begin{align*}
		\EE\bigg[\sum_{s \in \Siii} I_s(x_s)\bigg] = \EE\bigg[\sum_{s \in \Siii} l_s(q_s)\bigg] \leq \EE\bigg[\sum_{s=1}^{s_n} J_s l_s(q_s)\bigg] \leq \log(n) + \oO\big(\log(n)^{1/2}\EE[\log(s_n)^2]\big)
	\end{align*}
	Combing the bounds on the information ratio and the information gain, we get
	\begin{align*}
		\EE\bigg[\sum_{s \in S_3} \bar \Delta_s(\mu_s)\bigg]  \leq \big(c^* + \oO(\lambda^{1/2})\big)\big(\log(n) + \oO(\log(n)^{1/2} \EE[\log(s_n)^2])\big)
	\end{align*}
	Hence we conclude
	\begin{align*}
		\EE\B[\sum_{s \in \Siii} \Delta(s)\B] \leq c^* \log(n) + \oO\big(\lambda^{1/2} \log(n)\big)\,.
	\end{align*}
	Finally, with Lemma \ref{lem:effective-horizon}, we get that $\EE[\log(s_n)^b] \leq \oO(\log \log(n))$. Therefore, with $\lambda = \log(n)^{-2/3}$ all terms except for $c^* \log(n)$ are of lower order and the claim follows.
\end{proofof}

%% file: parts/appendix-technical.tex
\subsection{Technical Lemmas}

\begin{lemma}[Elliptic potential lemma] \label{lem:elliptic-potential} Assume that $\|x_{s}\|_{V_s^{-1}}^2 \leq 1$ and $\|x_s\|_2 \leq 1$. Then
	\begin{align*}
		\sum_{s=1}^{s_n} \|x_s\|_{V_s^{-1}}^2 \leq 2 \log \det(V_{s_n}) \leq 2 d \log\left(\frac{s_n + d}{d}\right)
	\end{align*}
\end{lemma}
\noindent A proof can be found in \citep[Lemma 11]{Abbasi2011improved}. Note that by $\diam(\xX) \leq 1$ and the choice $V_0 = \eye_d$, the assumptions of the lemma are always satisfied for our setting.

\begin{lemma}\label{lem:ms}
	Let $\beta_s = \|\hat \theta_s - \thetaopt\|_{V_s}^2$ and $m_s = \frac{1}{2}\min_{x \neq \hat x_s} \|\hat \nu(x) - \hat \theta_s\|_{V_s}^2$. Assume that $\beta_s < 2m_s$ and $\max_{x \in \xX}\Delta(x) \leq 1$. Then $\hat x_s = x^*$ and  further, for all $x \in \xX$,
	\begin{align*}
		\big((2m_s)^{1/2} - \beta_s^{1/2}\big)\|x^* -x\|_{V_s^{-1}} \leq 1\,.
	\end{align*}
\end{lemma}
\begin{proof}
	Since $m_s = \frac{1}{2}\min_{x \neq \hat x_s}\min_{\nu \in \cC_x} \|\nu - \hat \theta_s\|_{V_s}^2$, the assumption that $\beta_s = \|\hat \theta_s - \thetaopt\|_{V_s}^2 < 2 m_s$ implies that $\thetaopt \in \cC_{\hat x_s}$, and therefore $\hat x_s = x^*$. Further, for any $x \in \xX$,
	\begin{align*}
		0 = \min_{\nu : \|\nu - \hat \theta_s \|_{V_s}^2 \leq 2 m_s} \ip{\nu, x^* - x} &= \ip{\hat \theta, x^* - x} - (2m_s)^{1/2}\|x^* - x\|_{V_s^{-1}}\\
		& \leq \ip{\thetaopt, x^* - x} + (\|\hat \theta_s - \thetaopt \|_{V_s} - (2m_s)^{1/2})\|x^* - x\|_{V_s^{-1}}\,.
	\end{align*}
	Using $\Delta(x) = \ip{\thetaopt, x^* - x} \leq 1$ and rearranging completes the proof.
\end{proof}

\begin{lemma} \label{lem:exploration-regret}
	Let $\beta_s = \|\hat \theta_s - \thetaopt\|_{V_s}^2$ and define the indicator $B_s = \chf{\beta_{s,s^2} \geq \beta_s}$  for rounds $s$ where the confidence bounds at level $\beta_{s,s^2}$ are valid. Assume that $\max_{x \in \xX} \Delta(x) \leq 1$. Then 
	\begin{align*}
		R_n \leq \EE\left[\sum_{s=1}^{s_n} \Delta_s(x_s) B_s \right] + \oO\big(\log \log(n)\big) \,.
	\end{align*}
\end{lemma}
\begin{proof}
	Naturally, the regret decomposes into exploration and exploitation rounds. When $\beta_s > \beta_{s,s^2}$ (in exploration rounds, indexed by local time $s$) or $\beta_{s_t} > \beta_{s_t, t \log t}$ (in exploitation rounds, indexed by global time $t$), the parameter estimate is too inaccurate to bound the regret, and we simply bound $\Delta(x) \leq 1$. 
  On the other hand, in exploitation rounds where $\beta_{s_t} \leq \beta_{s_t, t \log t}$, by the definition of an exploitation round, it
  holds that $m_{s_t} \geq \beta_{s_t, t \log t} \geq \beta_{s_t}$ and by Lemma~\ref{lem:ms} this implies that $\hat x_s = x^*$ and the regret vanishes. Hence,
	\begin{align*}
		R_n = \sum_{t=1}^n \Delta(x_t) \leq \sum_{s=1}^{s_n} \Delta(x_s)B_s + \sum_{s=1}^{s_n} \chf{\beta_{s,s^2} < \beta_s} + \sum_{t=1}^n \chf{\beta_{s_t, t \log t} < \beta_{s_t}}
	\end{align*}
	Note that by \eqref{eq:confidence-coefficient}, we have  $\PP[\beta_{s,s^2} < \beta_s] \leq 1/s^2$ and $\PP[\beta_{s_t, t \log t} < \beta_{s_t}] < \frac{1}{t \log t}$. Hence, in expectation we get
	\begin{align*}
		\EE[R_n] &\leq \EE\left[\sum_{s=1}^{s_n} \Delta(x_s)B_s + \sum_{s=1}^{s_n} \frac{1}{s^2} + \sum_{t=1}^n \frac{1}{t \log t} \right]\\
		&\leq \EE\left[\sum_{s=1}^{s_n} \Delta(x_s)B_s\right] + \oO(\log \log(n))\,.
	\end{align*}
\end{proof}

\begin{lemma}\label{lem:effective-horizon}
	Assume that $\|x^*\|_2 > 0$. Then the number of exploration steps $s_n$ in Algorithm \ref{alg:ids-anytime} is bounded in expectation,
	\begin{align*}
		\EE[s_n^{1/2}] \leq \oO\left(d^2\Delta_{\min}^{-1}\log(n)^2 \|x^*\|_2^{-1}\right)\,.
	\end{align*}
	 In particular, for any $b \geq 1$, we have $\EE[\log(s_n)^b] \leq \oO(\log \log(n))$.
\end{lemma}
\begin{proof}
	By Theorem \ref{thm:regret-log},
	\begin{align*}
		\EE\left[\sum_{s=1}^{s_n} \delta_s\right] \leq \EE\left[\sum_{s=1}^{s_n} \hat \Delta_s(x_s)\right] \leq  \oO\left(d^2 \Delta_{\min}^{-1}  \|x^*\|_2^{-1} \log(n)^2\right)\,.
	\end{align*}
	We can assume that $2 \delta_s < \Delta_{\min}$, since there can be at most $\oO\left(d^2 \Delta_{\min}^{-2} \log(n)^2\right)$ steps where this condition is not satisfied. In particular, the assumption implies that $x^* = \hat x_s$, since for all $x \neq x^*$, $2 \hat \Delta_s(x) \geq \Delta_{\min}$. Therefore, 
	\begin{align*}
		\delta_s = \max_{z \in \xX}\ip{z - x^*, \hat \theta_s} + \beta_{s,s^2}^{1/2}\|z\|_{V_s^{-1}} \geq  \beta_{s,s^2}^{1/2} \|x^*\|_{V_s^{-1}} \geq \|x^*\|_2 s^{-1/2}\,.
	\end{align*}
	The last inequality follows from since $\lambda_{\max}(V_s) \leq s$. Hence further
	\begin{align*}
		\EE\left[\sum_{s=1}^{s_n} \delta_s\right] \geq \|x^*\| (s_n^{1/2} - \oO\left(d \Delta_{\min}^{-1} \log(n)^1\right))\,.
	\end{align*}
	This proves the first claim. For the second part, note that $\log(s)^b$ is concave for $s \geq \exp(b-1)$. Hence
\begin{align*}
	\EE[\log(s_n)^b] = 2^b\EE[\log(s_n^{1/2})^b] &\leq 2^b \EE[\log\big(\max(s_n^{1/2}, \exp(b-1))\big)^b]\\
	& \leq 2^b\log(\EE[s_n^{1/2}] + \exp(b-1))\\
	& \leq \oO(\log \log(n))
\end{align*}
\end{proof}


\begin{lemma}[Softmin approximation]\label{lem:softmin} $A_1, \dots A_k \geq 0$ be a sequence of positive numbers and $a = \min_{i \in [k]} A_i$. Let $q_i(x) \propto \exp(-\eta A_i)$ be exponential mixing weights with $\eta > 0$. Then
	\begin{align*}
		\sum_{i \in [k]} q_i A_i \leq a + \frac{\log(k)}{\eta}\,.
	\end{align*}
	Further, the mixing weights $q_i$ are bounded as follows,
	\begin{align*}
		\frac{1}{k} \exp\left(-\eta (A_i - a)\right) \leq q_i \leq \exp\left(-\eta (A_i - a)\right)\,.
	\end{align*}
\end{lemma}
\begin{proof}
	Let	$\psi^*_\eta(A) = \frac{1}{\eta} \log \left(\sum_{i \in [k]} \exp(\eta A_i))\right)$ be the Fenchel conjugate of the normalized entropy function. A direct calculation confirms that
	\begin{align*}
		q = \nabla_A \psi^*_{\eta}(-A)\,.
	\end{align*}
	By convexity of $\psi_\eta^*$,
	\begin{align*}
		\sum_i q_i A_i = \ip{\nabla \psi^*_\eta(- A), A} \leq \psi^*_\eta(0) - \psi^*_\eta(-A) \leq \tfrac{1}{\eta}\log(k) + \min_{i} A_i\,.
	\end{align*}
	The last inequality follows from
	\begin{align*}
		\psi^*_\eta(-A) = \eta^{-1} \log\Big(\sum_i \exp(-\eta A_i)\Big) \geq \eta^{-1} \log\Big(\exp(-\eta \min_i A_i)\Big) = - \min_i A_i\,.
	\end{align*}
	For the bound on the mixing weights, note that the claim is equivalent to the following bound on the normalization constant,
	\begin{align*}
		\exp(-\eta a) \leq \sum_{i} \exp(-\eta A_i) \leq k \exp(-\eta a)\,.
	\end{align*}
\end{proof}

\begin{lemma}[Convex Polytopes]\label{lem:convex-subspace}
	Let $K$ be a convex polytope. For unit vector $\eta \in \RR^d$, let
	$K_0 = \{x \in K : \ip{x, \eta} = 0\}$ be the intersection of $k$ with a $(d-1)$-dimensional hyperplane, which is assumed to be non-empty.
	Then there exists a constant $c > 0$ such that for all $z \in K$,
	\begin{align*}
		\min_{x \in K_0} \norm{x - z}_2 \leq c \ip{z, \eta}\,.
	\end{align*}
\end{lemma}

\begin{proof}
	Let $A = \{x \in K : \ip{x, \eta} \geq 0\}$, which is also a convex polytope. We first show there exists a $c > 0$ such that for all $z \in A$,
	\begin{align}
		\min_{x \in K_0} \norm{x - z}_2 \leq c \ip{z, \eta} \,.
		\label{eq:conv-upper}
	\end{align}
	The result follows by making a symmetric argument for $\{x \in K : \ip{x, \eta} \leq 0\}$.
	To establish (\ref{eq:conv-upper}),
	let $V \subset \RR^d$ be the vertices of $A$, which is a finite set.
	Define $h : A \setminus K_0 \to \RR$ by
	\begin{align*}
		h(z) = \max_{x \in K_0} \frac{\ip{\eta, z - x}}{\norm{z - x}}\,.
	\end{align*}
	Clearly, $1/c \triangleq \min_{v \in V : \ip{v, \eta} > 0} h(v) > 0$.
	Hence, the mapping $\varphi : V \to K_0$ such that $\varphi(v) = v$ for $v \in K_0$ and $\varphi(v) = \argmax_{x \in \kK_0} \frac{\ip{\eta, v - x}}{\norm{v - x}}$ satisfies $\norm{v - \varphi(v)}_2 \leq c\ip{\eta, v - \varphi(v)}$.
	Given any $z \in A$, let $\alpha$ be a probability distribution on $V$ such that $z = \sum_{v \in V} \alpha(v) v$ and let $x = \sum_{v \in V} \alpha(v) \varphi(v) \in K_0$.
	Then,
	\begin{align*}
		\norm{z - x}_2
		&= \norm{\sum_{v \in V} \alpha(v) v - \sum_{v \in V} \alpha(v) \varphi(v)}_2 \\
		&\leq \sum_{v \in V} \alpha(v) \norm{v - \varphi(v)}_2 \\
		&\leq c \sum_{v \in V} \alpha(v) \ip{\eta, v - \varphi(v)} \\
		&= c \ip{\eta, z}\,.
	\end{align*}
\end{proof}

\begin{lemma}\label{lem:ls-sherman}
	The one-step update to the least-squares estimator with data $y_{s} = \ip{x_{s}, \thetaopt} + \epsilon_{s}$ is
	\begin{align*}
		\hat{\theta}_{s+1} - \hat{\theta}_s = V_s^{-1} x_s\left(\frac{\epsilon_s + x_s^\T (\thetaopt - \hat{\theta}_s)}{1+ \|x_s\|_{V_s^{-1}}^2}\right)\,.
	\end{align*}
\end{lemma}
\begin{proof}
	The difference can be computed with the Sherman-Morrison-Woodbury formula,
	\begin{align*}
		\hat{\theta}_{s+1} - \hat{\theta}_s &= V_{s+1}^{-1}\sum_{i=1}^{s} x_iy_i -  \hat{\theta}_s \\
		&= V_s^{-1} \sum_{i=1}^{s-1} x_i y_i + V_s^{-1}x_sy_{s} -  \frac{V_s^{-1}x_sx_s^\T V_s^{-1}}{1 + \|x_s\|_{V_s^{-1}}^2}\sum_{i=1}^{s} x_i y_i - \hat{\theta}_s\\
		&= V_s^{-1}x_sy_{s}  
    -  \frac{V_s^{-1}x_s \|x_s\|^2_{V_s^{-1}} y_s}{1 + \|x_s\|_{V_s^{-1}}^2}
    -  \frac{V_s^{-1}x_sx_s^\T \hat \theta_s}{1 + \|x_s\|_{V_s^{-1}}^2} \\
		&= V_s^{-1} x_s\left(y_{s} - \frac{\|x_s\|_{V_s^{-1}}^2y_{s}}{1+ \|x_s\|_{V_s^{-1}}^2} - \frac{x_s^\T \hat{\theta}_s}{1+ \|x_s\|_{V_s^{-1}}^2}\right)\\
		&= V_s^{-1} x_s\left(\frac{y_{s} - x_s^\T \hat{\theta}_s}{1+ \|x_s\|_{V_s^{-1}}^2}\right)\\
		&= V_s^{-1} x_s\left(\frac{\epsilon_s + x_s^\T (\thetaopt - \hat{\theta}_s)}{1+ \|x_s\|_{V_s^{-1}}^2}\right)\,.
	\end{align*}
\end{proof}


%% file: parts/covering.tex
\section{Information-Directed Sampling as a Primal-Dual Method}\label{app:oracle} 

This section serves as self-contained exposition to establish the link between information-directed sampling and primal-dual approaches used to solve the lower bound \eqref{eq:lower-linear}. \emph{Note that in this section, quantities such as $\hat \Delta_t$, $\delta_t$ and $I_t$ are re-defined independently of the main text.}\medskip

For simplicity, for the remainder of this section we fix finitely many alternative parameters $\nu_1, \dots, \nu_l \in \mM$ for which $x^*(\nu) \neq x^*(\theta^*)$. Define constraint vectors $h_j \in \RR^\xX$ as $h_j(x) = \frac{1}{2}\ip{\nu_j - \thetaopt, x}^2$ for each $x \in \xX$ and $j = 1, \dots, l$. Our boundedness assumptions imply $\|h_j\|_2 \leq 1$.
With this notation, the lower bound \eqref{eq:lower-linear} can be written as a \emph{linear covering program},
\begin{align}
	c^* = \inf_{\alpha \in \RR_{\geq 0}^k} \sum_{x \in \xX} \alpha(x) \ip{x^* - x, \thetaopt} \qquad \text{s.t.} \qquad \forall j = 1, \dots, l, \quad  h_j(\alpha) \geq 1\,. \label{eq:covering-program}
\end{align} 
It is immediate from the assumption that $\xX$ spans $\RR^d$ that the program is feasible. Further, there is no cost for playing the optimal action $x^*$ since the corresponding gap is zero, $\Delta(x^*) = 0$. Following the terminology of \citet{jung2020crushoptimisim}, we refer to a constraint $h_j$ with $h_j(x^*) > 0$ as \emph{docile}. Such constraints are trivially satisfied by letting $\alpha(x^*) \rightarrow \infty$, while the regret from allocating $x^*$ remains zero in the limit. To simplify our exposition further, here we assume that there are \emph{no docile constraints}, i.e.\ $h_j(x^*) = 0$ holds for all $j=1,\dots, l$.

The objective of this section is to derive sequential strategies to solve \eqref{eq:covering-program} in the \emph{oracle setting}, where the exact cost and constraint vectors are known. Thereby, we set aside all complications that arise in the statistical setting. Specifically, we seek to incrementally determine a sequence of distributions $\mu_1, \dots \mu_n \in \sP(\xX)$ over actions, which define a cumulative allocation $\alpha_n = \sum_{t=1}^n \mu_t$. 
We say an allocation is \emph{asymptotically optimal} at rate $\beta_n$ if 
\begin{align}
\lim_{n\rightarrow \infty} \frac{\Delta(\alpha_n)}{\beta_n} = c^*, \quad  \text{and} \quad \forall j=1,\dots,l, \,\lim_{n\rightarrow \infty} \frac{h_j(\alpha_n) }{\beta_n} \geq 1 \,.
\end{align}
The lower bound suggests a choice which satisfies $\lim_{n \rightarrow \infty} \beta_n = \log(n)$.

\paragraph{Online Convex Optimization}
We review an approach due to \citet{garg2007faster,arora2012multiplicative}, which solves covering LPs -- such as the oracle lower bound -- using \emph{online convex optimization} (OCO). The same idea has recently inspired bandit algorithms for best arm identification \citep{degenne2019games} and regret minimization \citep{degenne2020structure}. The approach sets up a fictitious two-player game that converges to the saddle point of the Lagrangian,
\begin{align*}
 \max_{\lambda \geq \RR_{\geq 0}^l} \min_{\alpha \in \RR_{\geq 0}^k} \left\{\lL(\alpha, \lambda) = \Delta(\alpha) - \sum_{j=1}^l \lambda_j (h_j(\alpha) - 1)\right\}\,.
\end{align*}
It is easy to verify that strong duality holds, and we can interchange the maximum and minimum. Note that the dual variables are on an unbounded space, but it turns out that we can normalize them. The KKT conditions are
\begin{align*}
\Delta(x) - \sum_{j=1}^l \lambda_j h_j(x) &= 0 && \text{(stationarity)}\\
\lambda_j (h_j(\alpha) - 1) &= 0 && \text{(complementary slackness)}
\end{align*}
Combining both, we find that $c^* = \sum_{j=1}^l \lambda_j$. This implies that the optimal cost $c^*$ normalizes the dual variables $q_j = \lambda_j / c^*$. The normalized Lagrangian is
\begin{align}
\lL(\alpha, q) = \Delta(\alpha) - c^* \sum_{j=1}^l q_j (h_j(\alpha) - 1)\,,
\end{align}
where $q \in \sP([l])$ is a distribution over the constraints. Recall that the allocation $\alpha_n = \sum_{t=1}^n \mu_t$ is chosen sequentially. In each iteration of the game, the \emph{first player}, or $q$-learner, chooses a distribution $q_t \in \sP([l])$ over the constraints. Then, the response of the \emph{second player} is a distribution $\mu_t \in \sP(\xX)$ over actions, which defines the allocation $\alpha_n = \sum_{t=1}^n \mu_t$. The loss of the $q$-learner is defined by the second player's response $\mu_t$,
\begin{align}
l_t(q) = \sum_{j=1}^l q_{t}(j) h_j(\mu_t)\,,
\end{align}
which is linear in the dual variable $q_t$. The loss sequence defines the $q$-learner regret $\Lambda_n$ (not to be confused with $R_n$), which is
\begin{align}
\Lambda_n = \sum_{t=1}^n l_t(q_t) - \min_{q \in \sP([l])} \sum_{t=1}^n l_t(q)\,.
\end{align}
For concreteness, we choose the exponential weights learner \citep{vovk1990aggregating,littlestone1994weighted},
\begin{align*}
q_t(j) \propto \exp\left(-\eta_t \sum_{s=1}^{t-1} l_t(j) \right)
\end{align*}
with learning rate $\eta_t$. Standard regret bounds for online convex optimization guarantee $\Lambda_n \leq \oO(\sqrt{n})$ for suitably chosen learning rate schedules. More refined techniques lead to first-order regret bounds, which scale with the best loss in hindsight $\Lambda_n \leq \oO(\sqrt{\min_{i} \sum_{t=1}^n l_t(i)})$, see for example \citep{cesa-bianchi_improved_2006}. Given the choice $q_t$ of the $q$-learner, we define the combined constraint $I_t = \sum_{j=1}^l q_t(j) h_j$. The policy response is defined as 
\begin{align}
\mu_t = e_{x_t},\quad \text{where} \quad x_t = \begin{cases} 
\argmin_{x \in \xX\setminus x^*} \frac{\Delta(x)}{I_t(x)} & \text{if } \min_{j} \alpha_{t-1}^\T h_j < \beta_n \\
x_t = e_{x^*} & \text{else.}
\end{cases} \label{eq:OCO-choice}
\end{align}
The second case corresponds to \emph{exploitation}, which happens as soon as the constraints are satisfied:
\begin{align*}
\min_{j} h_j^\T\alpha_{t-1} = \min_j \sum_{x} \alpha_{t-1}(x) \ip{\nu_j - \thetaopt, x}^2 \geq \beta_n\,.
\end{align*} 
Note that $x^*$ is the only action which does not incur cost. On the contrary, when $\min_{j} h_j(\alpha_t) < \beta_n$, the policy allocates on the suboptimal action $x_t \neq x^*$ with the best cost/constraint ratio, $\min_{x \neq x^*} \Delta(x)/I_t(x)$. Since there are no docile constraints, we have $I_t(x^*) = 0$ and $\mu_t = e_{x_t}$ corresponds to the \emph{optimal} allocation for the rescaled linear program with the single combined constraint $I_t = \sum_{j=1}^l q_t(j) h_j$,
\begin{align*}
\min_{\alpha \in \RR_{\geq 0}^\xX} \Delta(\alpha) \quad \text{s.t.} \quad I_t(\alpha) \geq I_t(x_t)\,.
\end{align*}
Rescaling the optimal solution $\alpha^*$ to the original covering program \eqref{eq:covering-program}, we obtain an upper bound to the cost of choosing $\mu_t = e_{x_t}$,
\begin{align*}
\Delta(\mu_t) \leq \Delta(I_t(\mu_t) \alpha^*) = c^* I_t(\mu_t)\,.
\end{align*}
Since $I_t(\mu_t) = l_t(q_t)$, we can make use of the regret bound for the $q$-learner,
\begin{align}
\sum_{t=1}^n I_t(\mu_t) = \sum_{t=1}^n l_t(q_t) \leq \min_{j} \alpha_n^\T h_j + \Lambda_n \leq \beta_n + \oO\big(\beta_n^{1/2}\big)\,. \label{eq:info-gain-regret}
\end{align}
For the last inequality, we used that $\min_{j} \alpha_n^\T h_j \leq \beta_n + 1$ is guaranteed by the definition \eqref{eq:OCO-choice} and boundedness, $I_t(x) \leq 1$. Further, we assume a first-order regret bound $\Lambda_n \leq \oO(\beta_n^{1/2})$ for the $q$-learner. From here, we easily bound the regret $R_n$ of the allocation $\alpha_n$,
\begin{align*}
R_n = \ip{\alpha_n, \Delta} = \sum_{t=1}^{n} \Delta(\mu_t) \leq c^* \sum_{t=1}^{n} I_t(x_t) \leq c^* \beta_n + \oO\big(c^* \beta_n^{1/2}\big)\,.
\end{align*}
With some care, this approach can be translated to a bandit algorithm, by replacing all unknown quantities with statistical estimates, see \citet{degenne2020structure}. The formulation presented here differs from previous work in that it avoids a re-parametrization of the allocation, and the argument to bound the regret is more direct. We are now in the position to establish a link between information-directed sampling and the two-payer minimax game setup. 

\paragraph{Oracle Information-Directed Sampling}
For reasons that become apparent soon, we refer to the combined constraints $I_t = \sum_{j=1}^l q_t(j) h_j \in \RR_{\geq 0}^\xX$ as the \emph{information gain}, where $q_t$ is the output of the same $q$-learner as before. We also introduce a positive \emph{error} term $\delta_t > 0$ that is added to the gaps, to obtain \emph{approximate gaps} $\hat \Delta_t(x) = \Delta(x) + \delta_t$. This choice anticipates the definition for the gap estimate, which we use later in the bandit setting. Moreover, $\delta_t > 0$ avoids a degenerate regret-information trade-off and allows us to treat all actions in a unified manner.

Information-directed sampling approaches the regret minimization problem by sampling actions from a distribution that minimizes the \emph{information ratio},
\begin{align*}
\mu_t = \argmin_{\mu \in \sP(\xX)} \left \{\Psi_t(\mu) = \frac{\hat \Delta_t(\mu)^2}{I_t(\mu)}\right\}\,.
\end{align*}
We follow this strategy as long as $\min_j h_j(\alpha_t) < \beta_n$. Once the constraints are satisfied, we resort to playing the optimal action $x^*$ as before. This allows to bound the $q$-learner regret $\Lambda_n$, and therefore the total information gain as in \eqref{eq:info-gain-regret}.
We make the assumption that the estimation gap $\delta_t$ is small compared to the minimum gap $\Delta_{\min} = \min_{x \neq x^*} \Delta(x)$,
\begin{align}
2 \delta_t \leq \min_{x \neq x^*} \hat \Delta_t(x) \label{eq:cost-condition}
\end{align}
or equivalently, $\delta_t \leq \Delta_{\min}$. At first sight, the IDS distribution does not relate to the previous analysis, since the ratio appears with the cost squared. However, a closer inspection reveals a strong connection, which is summarized in the following lemma. 
\begin{lemma}
Let $\mu_t = \argmin_{\mu \in \sP(\xX)} \frac{\hat \Delta_t(\mu)^2}{I_t(\mu)}$ be the IDS distribution.
	If $2 \delta_t \leq \min_{x \neq x^*} \hat \Delta_t(x)$ and $I_t(x^*) = 0$, then $\mu_t = (1-p_t) e_{x^*} + p_t e_{z_t}$ with  alternative action $z_t = \argmin_{z \in \xX} \frac{\Delta(z)}{I_t(z)}$ and trade-off probability $p_t = \frac{\delta_t}{\hat \Delta_t(z_t) - \delta_t} = \frac{\delta_t}{\Delta(z_t)}$.
\end{lemma}
\begin{proof}
	Let $\psi(p) = \frac{((1-p) \hat \Delta_t(\mu_t) + p \delta_t)^2}{(1-p)I_t(\mu_t)}$ be the ratio obtained from shifting probability mass to $x^*$. By definition of the IDS distribution, we must have 
	\begin{align*}
	0 \leq \frac{d}{d p}\psi(p)|_{p=0} = \frac{2 \hat \Delta_t(\mu_t)\delta_t - \hat \Delta_t(\mu_t)^2}{I_t(\mu_t)}\,.
	\end{align*}
	Re-arranging shows that $\hat \Delta_t(\mu_t) \leq 2 \delta_t$. The IDS distribution can always be chosen with a support of at most two actions, which is a result by \citet[Proposition 6]{Russo2014learning}.  With the condition $2 \delta_t \leq \min_{x \neq x^*} \hat \Delta_t(x)$, it therefore suffices to optimize over distributions $\mu(p,z) = (1-p)e_{x^*} + p e_z$. A simple calculation reveals that $\argmin_{p \in [0,1]} \Psi_t(\mu(p,z)) = \frac{\delta_t}{\hat \Delta_t(z) - \delta_t}$, and 
	\begin{align*}
	\min_{\mu} \Psi_t(\mu) = \min_{z \neq x^*} \min_{p \in [0,1]} \Psi(\mu(p,z)) = \min_{z \neq x^*} \frac{4 \delta_t (\hat \Delta_t(z) - \delta_t)}{I_t(z)}\,.
	\end{align*}
	Therefore the alternative action is $z_t = \argmin_{z \neq x^*}\left\{\frac{\hat \Delta_t(z) - \delta_t}{I_t(z)} = \frac{\Delta(z)}{I_t(z)} \right\}$.
\end{proof}
The lemma shows that $\supp(\mu_t) = \{x^*, z_t \}$ and the alternative action $z_t \neq x^*$ minimizes the same cost-to-constraint ratio as before. Hence, almost the same argument implies a regret bound $R_n \leq c^* \beta_n + \oO(c^*\beta_n^{1/2})$.

Unlike the approach presented before, the distribution $\mu_t$ allocates mass to the zero-cost action, even before the constraint threshold $\beta_n$ is reached. Importantly, the randomization also allows to bound the regret in a worst-case manner. In the statistical setting, we expect that the estimation error roughly decreases at a rate $\delta_t \approx t^{-1/2}$. 
With the trade-off probability $p_t = \frac{\delta_t}{\Delta(z_t)}$ the expected cost per round is $\Delta(\mu_t) = \delta_t$. In other words, we get a finite-time, problem-independent bound on the regret $R_n$, 
\begin{align*}
R_n =  \Delta(\alpha_n) \leq \sum_{t=1}^n \delta_t \leq \oO(\sqrt{n})\,.
\end{align*}

Lastly, we link our analysis to the standard Cauchy-Schwarz argument that appears in all previous regret bounds for IDS (c.f.~\cite{Russo2014learning,Kirschner2018}). A direct calculation using the trade-off probability $p_t$ reveals that the expected approximate cost of the IDS distribution is exactly two times the actual cost and equals the estimation gap,
\begin{align*}
\frac{1}{2} \hat \Delta_t(\mu_t) = \delta_t  = \Delta(\mu_t)\,.
\end{align*}
Note that exact equality only holds because we have $I_t(x^*) = 0$ (no docile constraints). We continue to bound the regret with the Cauchy-Schwarz inequality and using the definition of the information ratio $\Psi_t = \frac{\hat \Delta_t(\mu_t)^2}{I_t(\mu_t)}$,
\begin{align*}
R_n = \sum_{t=1}^n \Delta(\mu_t) = \frac{1}{2} \sum_{t=1}^n \hat \Delta_t(\mu_t)
 \leq \frac{1}{2} \sqrt{\sum_{t=1}^n \Psi_t \sum_{t=1}^n I_t(\mu_t)} \,.
\end{align*}
Let $\tilde \alpha^* = \alpha^* \chf{x \neq x^*}$ be the optimal allocation \eqref{eq:covering-program} restricted to suboptimal actions. Note that by definition and the fact that we excluded docile constraints, $I_t(\tilde\alpha^*) \geq 1$ and $\Delta(\tilde \alpha^*) = c^*$. Define a distribution $\mu(p) = (1-p)e_{x^*} + p \tilde \alpha^*/\|\tilde \alpha^*\|_1$, which randomizes between the best action and optimal allocation with trade-off probability $p$. A simple calculation reveals that,
\begin{align*}
 \Psi_t = \min_\mu \frac{\hat \Delta_t(\mu)^2}{I_t(\mu)} \leq \min_{p}  \frac{\hat \Delta_t (\mu(p))^2}{I_t(\mu(p))} \leq \frac{4 \delta_t \Delta(\alpha^*)}{I_t(\tilde \alpha^*)} \leq 4 c^* \delta_t\,.
\end{align*}
We combine the inequality and the regret bound for the $q$-learner to get 
\begin{align*}
R_n = \sum_{t=1}^n \Delta(\mu_t) \leq  \frac{1}{2} \sqrt{\sum_{t=1}^n 4 c^* \delta_t \left(\beta_n + \oO(\beta_n^{1/2})\right)} \,.
\end{align*}
Squaring both sides, using again that $\delta_t = \Delta(\mu_t)$ and solving for the regret yields the desired bound,
\begin{align*}
R_n = \sum_{t=1}^n \Delta(\mu_t) \leq c^*  \beta_n + \oO(c^*\beta_n^{1/2}) \,.
\end{align*}

%% file: parts/appendix-bayesian.tex
\section{Approximating Mutual Information}\label{app:bayesian}

The information gain function that was primarily analyzed in the Bayesian framework by \citet{Russo2014learning} is  the mutual information
\begin{align*}
	I_t^\IMI(x) = \II_t(y_t; x^*|x_t=x) = \HH_t(x^*) - \HH_t(x^* |y_t,x_t=x)\,.
\end{align*}
The second equality rewrites the mutual information as the entropy reduction on $x^*$, which is a random variable in the Bayesian setting. Computation of the posterior distribution is tractable with a Gaussian prior $\nN(0, \lambda^{-1})$ on the parameter and Gaussian observation likelihood $y_t \sim \nN(\ip{x_t, \theta}, 1)$. In this case the posterior distribution is $\nN(\hat \theta_t, V_t^{-1})$. However, computing the mutual information requires further evaluations of $d$-dimensional integrals which is challenging even with Gaussian distributions.

As a remedy, \citet{Russo2014learning} proposed a \emph{variance-based} information gain
\begin{align}
	I_t^\IVAR(x) \eqdef \EE_t[\big(\EE_t[\ip{x,\theta}|x^*] - \EE_t[\ip{x,\theta}]\big)^2] = \EE_t[\ip{\bar \nu_t(x^*) - \htheta_t, x}^2]\,.\label{eq:info-var-def}
\end{align}
The last step uses that $\EE_t[\theta] = \htheta_t$ and we defined $\bar \nu_t(x) = \EE_t[\theta|x^*=x]$. They further showed that the variance-based information gain lower-bounds the mutual information, $I_t^\IMI(x) \geq 2 I_t^\IVAR(x)$, while, at the same time, the information ratio is still bounded in the Bayesian setting with linear reward \citep[Proposotion 7]{Russo2014learning}. Importantly, \eqref{eq:info-var-def} can be approximated for a moderate number of actions using samples from the posterior distribution.

We compute the posterior probability $\bar q_t(c) \eqdef \PP_t[x^* = z]$ with a Laplace approximation of the integral over the cell $\cC_z = \{\theta \in \mM: x^*(\theta) = z \}$,\looseness=-1
\begin{align*}
	\bar q_t(z) &= \frac{1}{\sqrt{(2\pi)^{d} \det(V_t)}} \int_{\cC_z} \exp\left(-\tfrac{1}{2}\|\nu - \hat \theta_t\|_{V_t}^2\right) d\nu \approx Q_z^{-1} \exp\left(-\tfrac{1}{2}\|\tilde \nu_t(z) -  \hat \theta_t\|_{V_t}^2\right)\,,
\end{align*}
where $\tilde \nu_t(x) = \argmin_{\nu \in \cC_x} \|\nu - \hat \theta_s\|_{V_s}^2$.
Similarly, in the Laplace limit, the conditional distribution $\PP_t[\theta|x^*=x]$ concentrates on $\tilde \nu_t(x)$, which allows us to approximate $\bar \nu_t(x) \approx \tilde \nu_t(x)$. This leads to
\begin{align*}
	I_t^\IVAR(x) \approx \sum_{z \neq x^*} \bar q_t(z) \ip{\tilde \nu_t(x) - \hat \theta_t, x}^2\,,
\end{align*}
which resembles the definition of the cell-based information gain in \eqref{eq:info-gain-cell}.

Using the Laplace argument, we can also compute the mutual information more directly. Assuming that the posterior is well-concentrated, there exists an action $\bar x_t^*$ with $\bar q_t(\bar x_t^*) \approx 1$. For all $z \neq \bar x_t^*$ and interpolation variable $\tau \in [0,1]$, we define the conditional weights
\begin{align*}
	\bar q_t^\tau(z|x) \eqdef \bar q_t(z) \exp\left(-\tfrac{\tau}{2}\ip{\tilde \nu_t(z) - \htheta_t, x}^2\right)\,,
\end{align*}
and $q_t^\tau(\bar x_t^*|x) \eqdef 1 - \sum_{z \neq \bar x_t^*} q_t^\tau(z|x)$. Using the approximate posterior probabilities, the entropy reduction up to first order is
\begin{align*}
	\II_t(y_t;x^*|x_t=x) &\approx -\sum_{z \in \aA} \bar q_t(z) \log \bar q_t(z) + \sum_{z \in \aA} \big(\bar q_t^\tau(z|x) \log (\bar q_t^\tau(z|x)\big)\big|_{\tau=1}\\
	&\approx \sum_{z \in \aA} \frac{d}{d\tau} \big(\bar q_t^\tau(z|x) \log (\bar q_t^\tau(z|x)\big)\big|_{\tau=1}\\
	&= - \frac{1}{2}\sum_{z \neq \bar x_t^*}  \bar q_t(z) \ip{\nu_z - \theta, x}^2 \log\left(\frac{\bar q_t(z)}{1 - \sum_{z' \neq \bar x_t^*} \bar q_t(z')}\right)\,.
\end{align*}
Using that $-x \log x \geq x$ for $x \ll 1$, the last expression can be lower bounded to arrive at a form similar to the cell-based information gain \eqref{eq:info-gain-cell}. 

While our reasoning here is rather informal, we think that it warrants a more formal investigation in the future. Such results could be fruitful in two directions. First, interpreting the mutual information as an approximation of a dual loss could lead to an instance-dependent analysis for the Bayesian IDS algorithm, either on the frequentist or Bayesian regret. Second, the Bayesian information gain might serve as a starting point to design more effective information gain functions in the frequentist framework, for example adapted to other likelihood functions and regularizers.

%% file: parts/appendix-experiments.tex
\begin{figure}[t]
	\centering	\includegraphics{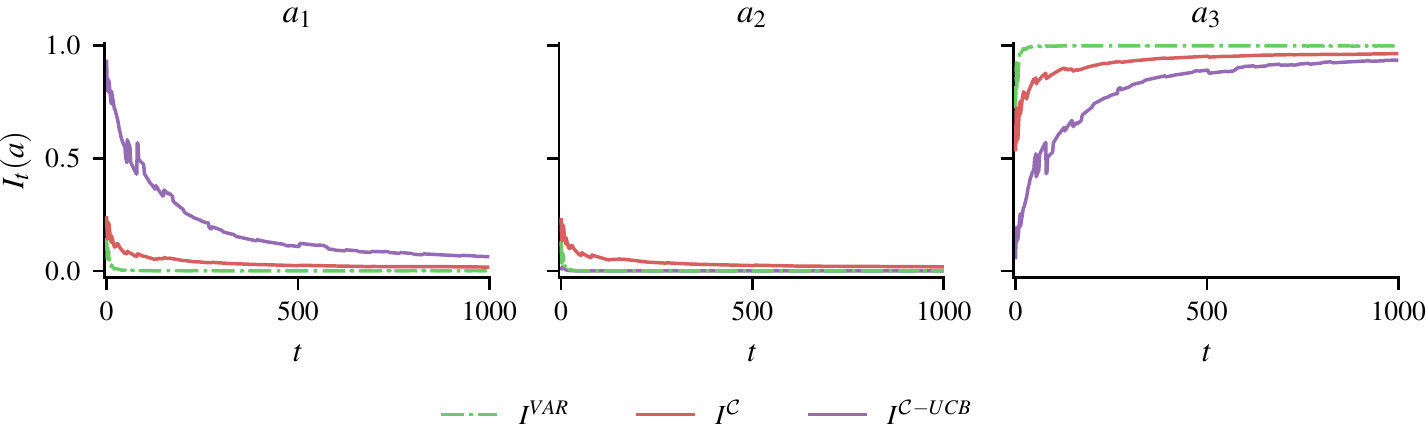}
	\caption{Comparison of information gain functions on the `end of optimism' example with $\epsilon=0.01$. The information gain functions are evaluated on the same trajectory generated by IDS-$I_t^\IAUCB$ and normalized such that ${\sum_{x \in A} I_t(x) = 1}$. On this instance, $x_1$ is optimal, $x_2$ is $\epsilon$-suboptimal, and $x_3$ is 1-suboptimal, but asymptotically more informative than action $x_2$. Clearly visible is that the lower-order terms of the $I_t^\IA$ and $I_t^\IAUCB$ are increasingly dominated by the asymptotic term where $x_3$ is the most informative action. $I_t^\IVAR$ is approximated using $10^4$ samples from the posterior distribution, and converges much faster than the information gain functions based on the $q$-learner, which uses a more conservative learning rate. Note that the approximation with posterior samples is unstable on a larger horizon without increasing the number of samples accordingly. } \label{fig:infogain}
\end{figure}

\begin{figure}
	\centering
	\includegraphics{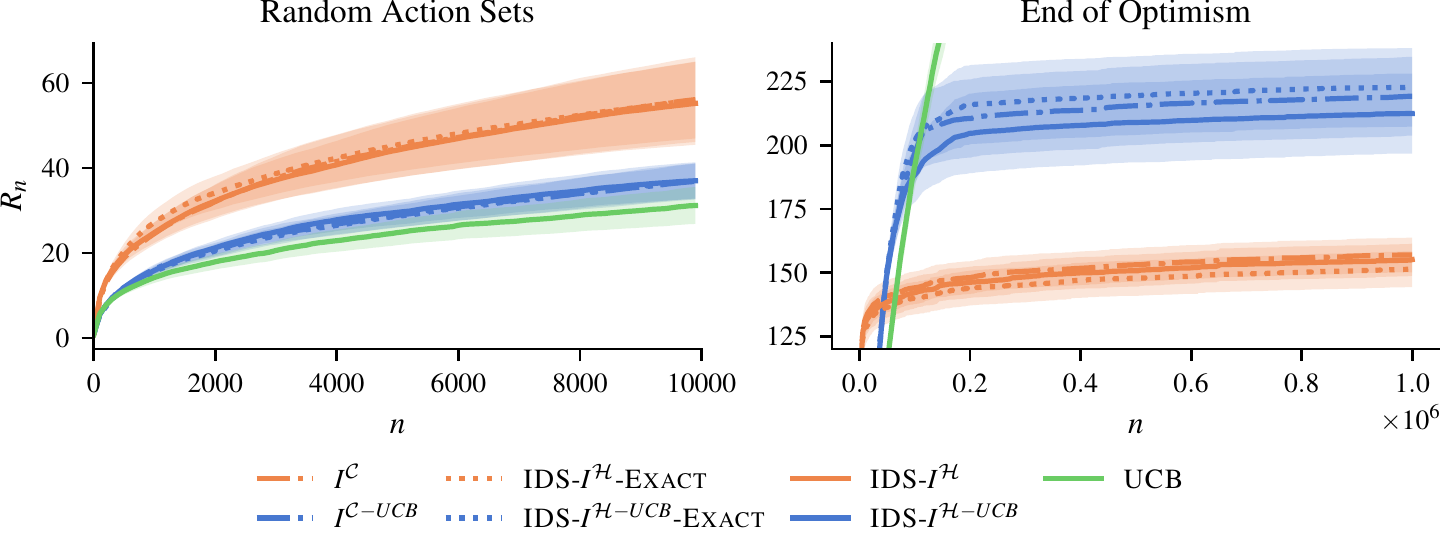}
	\caption{Comparison of information gain functions defined on cells and halfspaces respectively, as well as exact and approximate sampling from the IDS distribution. All variants achieve similar performance within the standard error, however the correction term has a larger impact on the regret. In the right plot, the y-axis is scaled to make the difference visible.  }
	\label{fig:ids-cell}
\end{figure}

\section{Additional Experiments}
\label{ap:experiments-supp}

In this section we summarize further numerical results. In Section \ref{app:comparison-info} we compare different information gain functions and show evidence that the cell based information gain \IDS-$I_s^\IACELL$ variant from Eq.\eqref{eq:info-gain-cell} behaves similarly to \IDS{} despite much longer runtimes. In Section \ref{app:comparison-hyper} we show the effect of the choice of confidence coefficient $\beta_t$ and the learning rate $\eta_s$ on the performance, and also evaluate the confidence coefficient derived by \cite{tirinzoni2020asymptotically}. In Section \ref{app:comparison-bayesian},  we provide a benchmark with Bayesian methods including Bayesian IDS and a runtime evaluation.

\subsection{Comparison of Information Gain Functions}\label{app:comparison-info}

The information gain functions used in the experiments are summarized below.
\begin{itemize}
	\item The information gain defined in the main text with halfspaces-based alternatives \eqref{eq:I-def}:
	\[I_s^\IA(x) = \sdfrac{1}{2}\sum_{z \neq \hat{x}_{s}} q_s(z) \left(|\ip{\hat \nu_{s}(z) - \htheta_{s},x}| + \beta_{s,s^2}^{1/2}\|x\|_{V_s^{-1}}\right)^2\]
	\item As before, but with correction only for the UCB action \eqref{eq:info-gain-ucb}:
	\[I_{s}^\IAUCB(x) = \sdfrac{1}{2}\sum_{z \neq \hat{x}_{s}} q_s(z) \left(|\ip{\hat \nu_{s}(z) - \htheta_{s},x}| + \chf{x = x_s^{\UCB}} \beta_{s,s^2}^{1/2}\|x\|_{V_s^{-1}}\right)^2\]
	\item The information gain defined with cell-based alternatives \eqref{eq:info-gain-cell}:
	\[I_s^\IACELL(x) \deq \sdfrac{1}{2}\sum_{z \neq \hat{x}_{s}} \tilde q_s(z) \left(|\ip{\tilde \nu_{s}(z) - \htheta_{s},x}| + \beta_{s,s^2}^{1/2}\|x\|_{V_s^{-1}}\right)^2\]
	\item The information gain defined on cells and UCB correction:
	\[I_{s}^\IAUCBCELL(x) = \sdfrac{1}{2}\sum_{z \neq \hat{x}_{s}} \tilde q_s(z) \left(|\ip{\tilde \nu_{s}(z) - \htheta_{s},x}| + \chf{x = x_s^{\UCB}} \beta_{s,s^2}^{1/2}\|x\|_{V_s^{-1}}\right)^2\]
	\item The variance-based information gain defined in \eqref{eq:info-var-def} and used for Bayesian IDS:
	\[	I_t^\IVAR(x) = \EE_t[\big(\EE_t[\ip{x,\theta}|x^*] - \EE_t[\ip{x,\theta}]\big)^2]\]
\end{itemize}
Alternative definitions of the information gain function based on the log-determinant potential are proposed by \citet{kirschner20partialmonitoring}. The resulting IDS algorithm satisfies similar worst-case guarantees but does not achieve asymptotic optimality, e.g.~on the end of optimism example.

Figure \ref{fig:infogain} shows a quantitative comparison of the information gain functions evaluated on the same trajectory on the end of optimism example. The asymptotic information gain based on half-spaces is not shown since it was empirically indistinguishable from the cell-based variant (which might be also due to the fact that there are only three cells in this example). This finding is confirmed by the regret plot in Figure \ref{fig:ids-cell}, where compare information gain functions, as well as the approximate IDS distribution (optimized directly on $\hat x_s$ and one other action) and the exact IDS distribution. The results show that, at least on our examples, there is almost no difference between the information gain defined with $\hat \nu_s$ and $\tilde \nu_s$, and the approximate and exact IDS sampling.

\subsection{Choice of Confidence Coefficient and Learning Rate} \label{app:comparison-hyper}

We run all our experiments with the simplified rate $\beta_t=\sigma^2(2\log(t)+d\log\log(t))$ instead, as suggested in \citet{tirinzoni2020asymptotically}. These result are shown on Figure~\ref{fig:all_plots_loglog} and confirm the statement in Section~\ref{sec:experiments} that there is no significant difference in the conclusions. However tuning $\beta_t$ to minimize regret \emph{significantly} improves the performance as shown in Figures \ref{fig:asymptotics-tuning-random} and \ref{fig:asymptotics-tuning-eoo}. On the other hand, tuning the learning rate $\eta_s$ has much less effect on the regret. The choice $\eta_s = 1/\sqrt{\beta_s}$ as suggested by the theory leads to good results and can be used to reduce the number of tuning parameters.

\begin{figure}[t]
	\centering
	\includegraphics{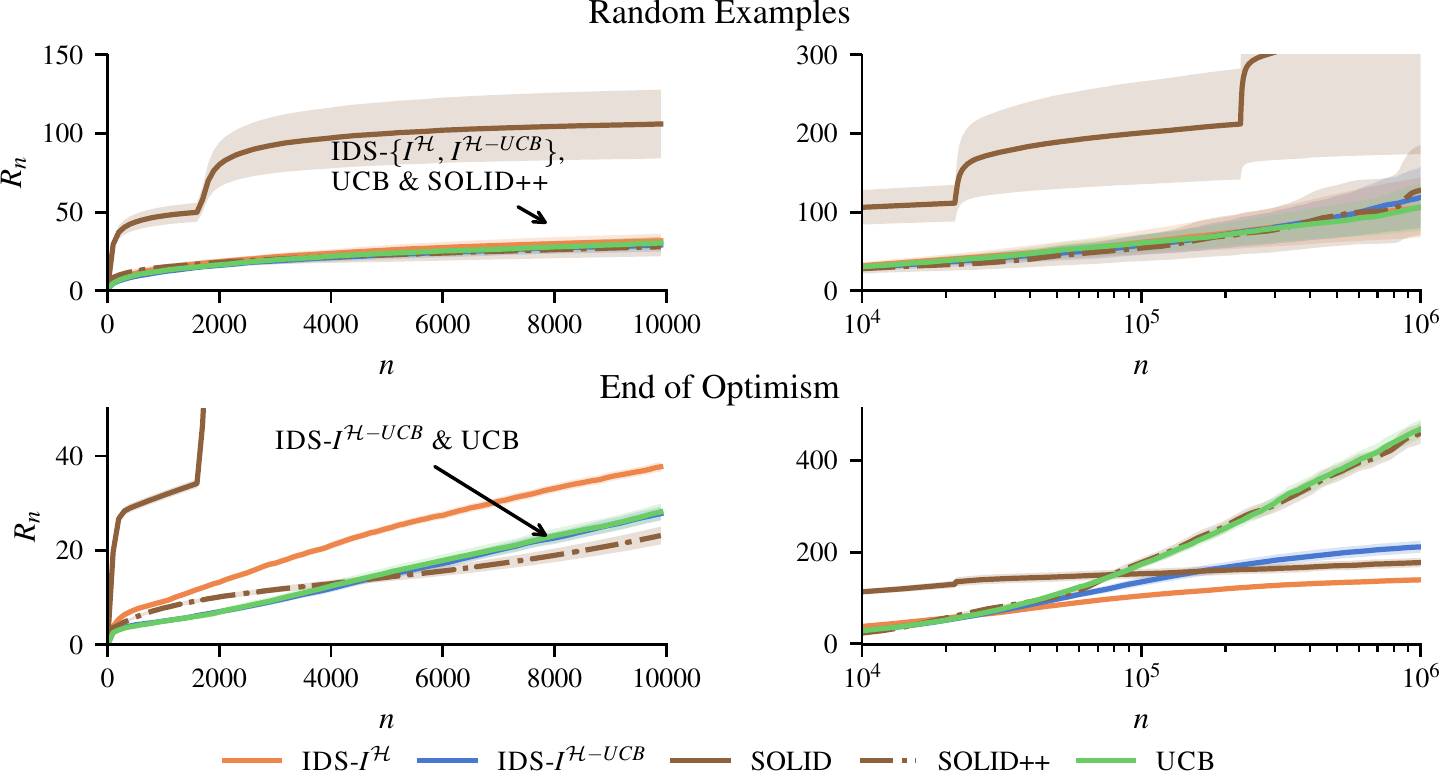}
	\caption{Experiments with $\beta_t=\sigma^2(2\log(t)+d\log\log(t))$. The numerical performance is comparable to the log-determinant confidence coefficient used in the main paper. }
	\label{fig:all_plots_loglog}
\end{figure}

\begin{figure}[p]
	\centering	\includegraphics{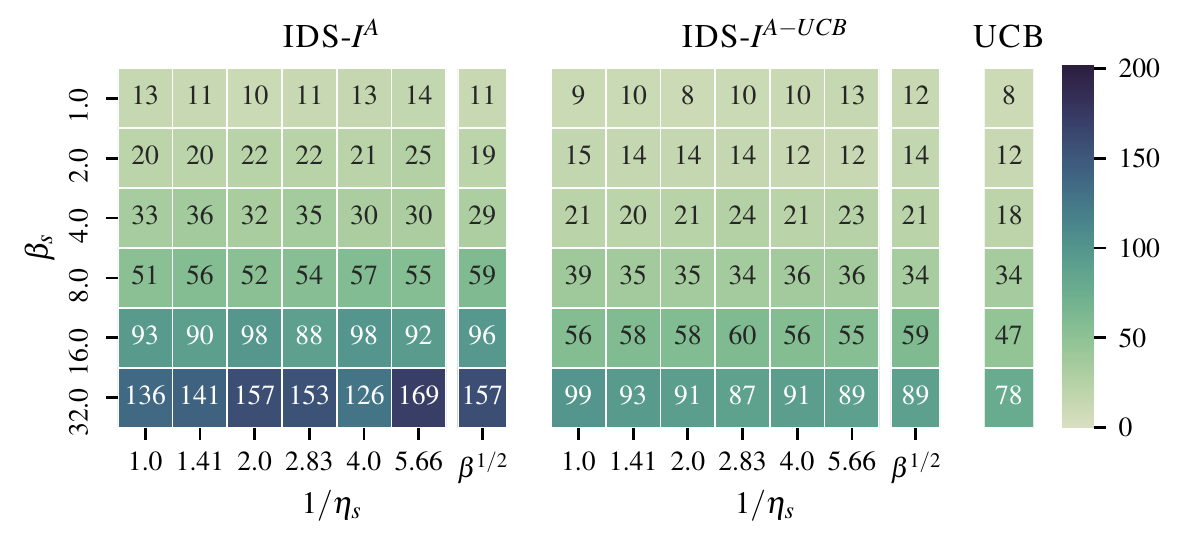}
	\caption{The matrix shows the regret on randomly generated action sets after $n=10^4$ steps for different values of $\beta_s$ and $\eta_s$. The first observation is that the regret can be \emph{significantly} reduced by choosing a smaller value for $\beta_s$. On the other hand, tuning the $q$-learning rate $\eta_s$ affects performance marginally. Tuning \emph{only} $\beta_s$ and setting $\eta_s = 1/\sqrt{\beta_s}$ as suggested by the theory leads to near optimal results. }\label{fig:asymptotics-tuning-random}
\end{figure}

\begin{figure}[p]
	\centering	\includegraphics{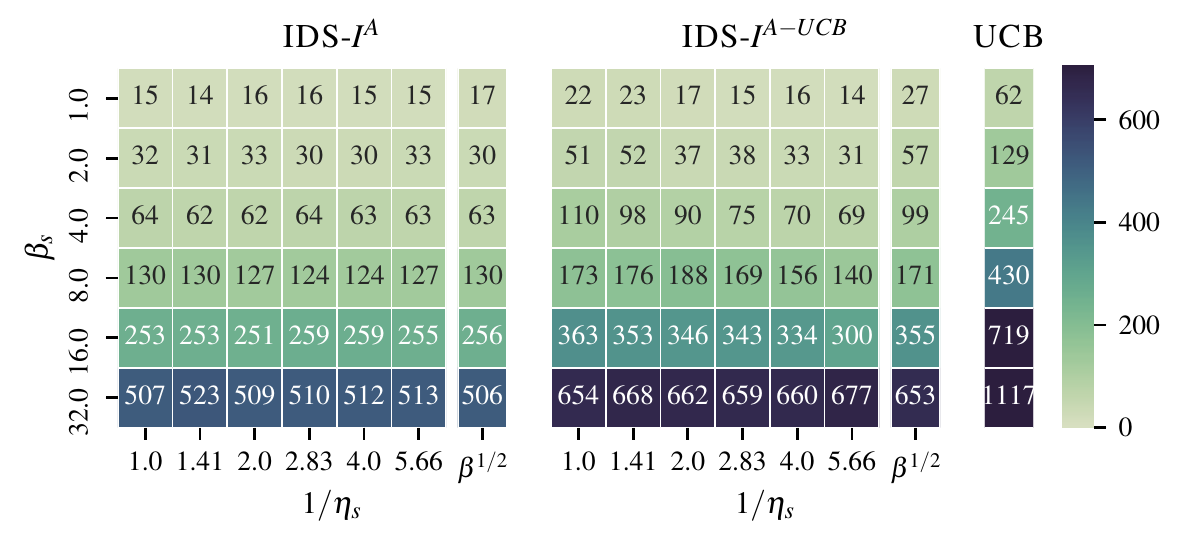}
	\caption{The matrix shows the regret on the `end of optimism' example after $n=10^6$ steps for different values of $\beta_s$ and $\eta_s$. The observations are similar as in Figure \ref{fig:asymptotics-tuning-random}.  Note that IDS is consistently better than UCB for any value of $\beta_s$.}\label{fig:asymptotics-tuning-eoo}
\end{figure}

\subsection{Comparison with Bayesian Methods} \label{app:comparison-bayesian}

In our last empirical benchmark, we include Bayesian methods, specifically Thompson sampling (\TS{}) and an approximation of Bayesian IDS. Our implementation of Bayesian IDS uses the variance-based information gain defined in \eqref{eq:info-var-def}, and we approximate the Bayesian gap estimates and information gain using $10^4$ posterior samples per round as suggested in \citep[Algorithm 6]{Russo2014learning}. The performance plots are in Figure \ref{fig:asymptotics-bayesian}. Thompson sampling significantly outperforms \textsc{UCB} and the frequentist IDS variants, unless we set $\beta_s=1$, which, as noted before, improves performance of the frequentist methods. The approximation of Bayesian IDS is the most effective on our benchmark, outperforming the best frequentist method on the `end of optimism' example roughly by a factor two. Lastly, we show runtime of all methods on a horizon $n=10^6$ in Table~\ref{tbl:runtime}. Note that despite the approximation, Bayesian IDS is computationally much more demanding, whereas the frequentist IDS is only about a factor of 5 slower than Thompson sampling on instances in $\RR^5$ with $k=50$ actions.

\begin{figure}[p]
	\includegraphics{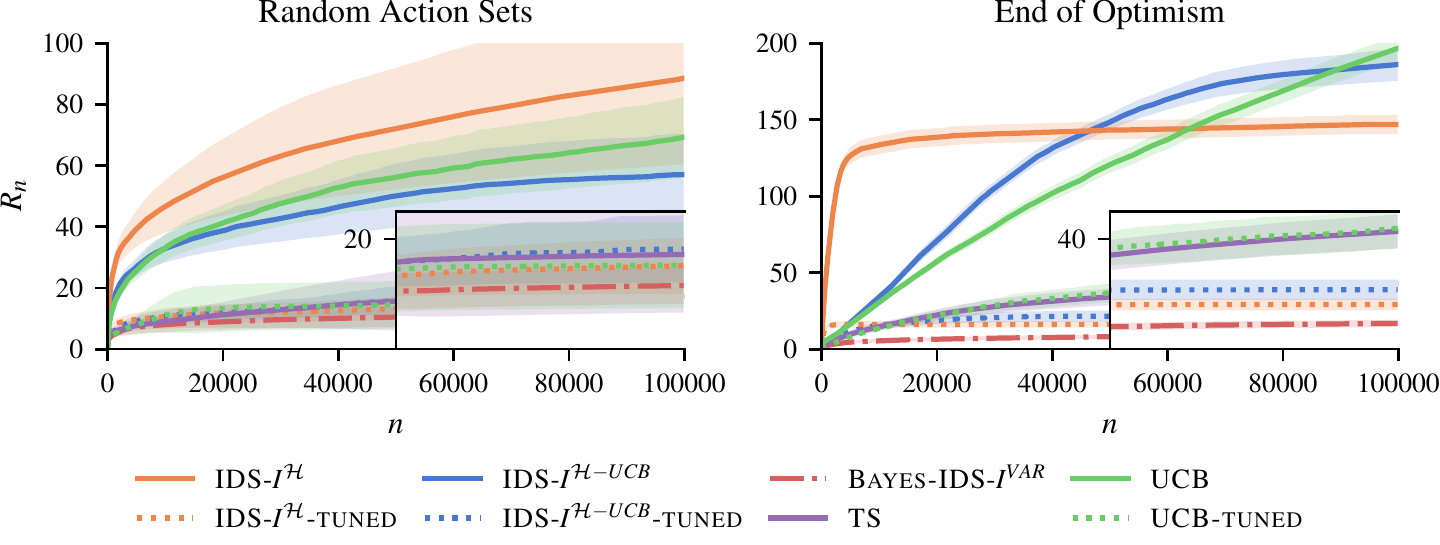}
	\caption{Comparison with Bayesian methods. On these examples, Bayesian IDS outperforms the frequentist methods, even when tuning the frequentist counterpart ($\beta_s=1$).}\label{fig:asymptotics-bayesian}
\end{figure}
\begin{table}[p]
	\begin{center}\noindent\begin{tabular}{lr@{\hskip0pt}lr@{\hskip0pt}l}
			\toprule%
			Algorithm & \multicolumn{2}{c}{$d=2,k=6$} & \multicolumn{2}{c}{$d=5,k=50$}\\
			\midrule%
			\textsc{Bayes-IDS}-$I^\IVAR\textsc{-Exact}$ & $561.7$&$\pm58.8$ & $2560.0$&$\pm78.4$\\
			\textsc{Bayes-IDS}-$I^\IVAR$ & $544.4$&$\pm69.7$ & $1771.9$&$\pm40.5$\\
			\textsc{IDS}-$I^\IAUCB\textsc{-Exact}$ & $50.5$&$\pm22.6$ & $798.5$&$\pm233.5$\\
			\textsc{IDS}-$I^\IAUCB$ & $45.7$&$\pm18.8$ & $106.8$&$\pm28.6$\\
			\textsc{UCB} & $26.9$&$\pm7.7$ & $23.9$&$\pm5.7$\\
			\textsc{TS} & $21.6$&$\pm5.9$ & $22.2$&$\pm6.9$\\
			\bottomrule
		\end{tabular}
	\end{center}
	\caption{Runtime comparison on random action sets with horizon $n=10^5$. The table shows mean and standard-deviation of the runtime in seconds on 50 runs, computed on a single core at 2.30GHz. The \textsc{Exact}-suffix indicates that the IDS distribution is computed exactly, whereas no suffix means that we minimize the tradeoff directly between $\hat x_s$ and an informative action as discussed at the end of Section \ref{sec:ids}.}\label{tbl:runtime}
\end{table}